\documentclass[accepted]{uai2026} 
\usepackage[american]{babel}

\usepackage[numbers]{natbib} 
\usepackage{mathtools} 
\usepackage{booktabs} 
\usepackage{tikz} 

\usepackage{hyperref}
\usepackage{url}
\usepackage{amssymb}
\usepackage{mathrsfs}
\usepackage{amsmath,amsfonts, amsthm, mathtools}
\usepackage{thmtools}
\usepackage{bm}
\usepackage{calc}
\usepackage[nameinlink,noabbrev,capitalize]{cleveref}
\usepackage{graphicx}
\usepackage{xparse}
\usepackage{pgffor}
\usepackage[ruled,vlined,linesnumbered]{algorithm2e}
\usepackage{algpseudocode}
\usepackage{textcomp}
\usepackage{booktabs}
\usepackage{pgfplots}
\usepackage{xcolor}
\usepackage{array}
\usepackage{enumitem}
\usepackage{subfiles}
\usepackage{pdflscape}
\usepackage{subcaption}
\usepackage{blkarray} 
\usepackage{nicematrix}
\usepackage{tikz}
\usetikzlibrary{decorations.pathreplacing}
\usetikzlibrary{calc, arrows.meta}
\usetikzlibrary{arrows.meta,patterns,backgrounds} 
\usepackage{tikz-3dplot}
\usepackage[absolute,overlay]{textpos}  
\usepackage{wrapfig}
\usepackage[mode=buildnew]{standalone}
\usepackage{longtable}   
\usepackage[most]{tcolorbox}
\usepackage[normalem]{ulem}
\usepackage{ifthen}
\pgfplotsset{compat=1.18}
\usepackage{float}

\newtheorem{theorem}{Theorem}[section]       
\newtheorem{lemma}[theorem]{Lemma}           
   
\newtheorem{definition}[theorem]{Definition}

\newtheorem{assum}[theorem]{Assumption}

\newcommand{\ru}{\mathrm{u}}
\newcommand{\re}{\mathrm{e}}
\newcommand{\rx}{\mathrm{x}}
\newcommand{\ry}{\mathrm{y}}

\newcommand{\tA}{\tilde{A}}
\newcommand{\tB}{\tilde{B}}
\newcommand{\abtildenobracket}{\tA, \tB}
\newcommand{\exptuple}{\rx_0, \ru}

\newcommand{\calK}{\mathcal{K}}

\definecolor{tumblue}{HTML}{0065BD}

\newcommand{\brn}[1]{\mathbb{R}^{#1}}

\newcommand{\calU}{\mathcal{U}}

\newcommand{\abpar}{(A, B)}
\newcommand{\abtilde}{(\tilde{A}, \tilde{B})}

\newcommand{\sphalf}{\hspace{0.05cm}}
\newcommand{\spone}{\hspace{0.1cm}}

\newcommand{\spbar}{\sphalf | \sphalf}

\newcommand{\bigbra}[1]{\left(#1\right)}

\newcommand{\rank}{\operatorname{rank}}
\renewcommand{\dim}{\operatorname{dim}}
\newcommand{\traj}[2]{\varphi\bigbra{#1\spbar #2}}

\newcommand{\inv}[1]{#1^{-1}}

\newcommand{\rightbarwrap}[2]{\left. #1 \right|_{#2}}

\newlength{\AcolOne}\newlength{\AcolTwo}
\newlength{\ArowTop}\newlength{\ArowBot}
\newlength{\tmpH}\newlength{\tmpD}

\newlength{\Bcol}\newlength{\BrowTop}\newlength{\BrowBot}
\newlength{\BtmpH}\newlength{\BtmpD}

\definecolor{violet}{HTML}{9527F5}

\newcommand{\Kcore}{[\rx_0 \; B]}
\newcommand{\vissub}{V(\rx_0)}
\newcommand{\calE}{\mathcal{E}}
\newcommand{\frob}[1]{\| #1 \|_F}
\newcommand{\gram}{\mathcal{G}_{[0,T]}}

\title{Limits of Learning Linear Dynamics from Experiments}

\author[1,2]{\href{mailto:<aybuke.ulusarslan@tum.de>}{Aybüke Ulusarslan}\thanks{Correspondence to: \href{mailto:aybuke.ulusarslan@tum.de}{aybuke.ulusarslan@tum.de}}}
\author[1,2,3]{\href{mailto:<niki.kilbertus@helmholtz-munich.de>}{Niki Kilbertus}{}} 
\author[1,2,3]{\href{mailto:<nora.schneider@helmholtz-munich.de>}{Nora Schneider}{}} 

\affil[1]{%
    Technical University of Munich 
}
\affil[2]{%
    Helmholtz Munich 
}
\affil[3]{%
    Munich Center for Machine Learning (MCML)
}

\begin{document}
\maketitle
\pagestyle{numbered}

\begin{abstract}
Learning governing dynamics from data is a common goal across the sciences, yet it is only well-posed when the underlying mechanisms are identifiable. In practice, many data-driven methods implicitly assume identifiability; when this assumption fails, estimated models can yield spurious predictions and invalid mechanistic conclusions. Classical identifiability guarantees for controlled linear time-invariant (LTI) systems provide sufficient conditions-- controllability and persistent excitation-- but leave open whether identifiability holds when these conditions fail, and which parts of the system remain identifiable without full identifiability.
We show that the experimental setup, i.e., the realized initial state and control input, dictates a fundamental limit on the information recoverable from the observed trajectory. We develop a geometric characterization of this limit and derive a closed-form description of all systems consistent with the experimental setup. Crucially, we prove that even when the full system is not identifiable, the restricted dynamics on the subspace reachable by the experiment remain uniquely determined.
\end{abstract}

\section{Introduction}\label{sec:intro}

Recovering the dynamics of a system from data is a central goal across sciences and engineering \citep{perakis2023covid, somacal2022uncovering, nelson1998flight}.
Ordinary differential equations (ODEs) are widely used to model such dynamical systems.
Recent work proposes to learn governing ODEs directly from time-series data, ranging from neural ODEs \citep{chen2018neural} to symbolic regression approaches \citep{brunton2016discovering}.
When learning dynamics from time-series data, a fundamental prerequisite is that the underlying dynamics are identifiable \citep{BELLMAN1970329}. When identifiability fails, multiple systems can generate identical trajectories. In this case, an identification algorithm will return one of these systems that could have generated the data, often with no indication that it is not unique. As a consequence, the estimated system may extrapolate incorrectly to new initial conditions and control inputs, or lead to spurious mechanistic interpretations of the underlying dynamics (see \Cref{fig:non-identifiability}).

Classical identifiability results for LTI systems \citep{kalman1963mathematical,willems2005note,BELLMAN1970329} establish that controllability of the system and a persistently exciting (PE) input are sufficient to uniquely recover the underlying dynamics from a single experiment. As a result, controllability is frequently treated as a default assumption. Yet, in many practical settings---especially sparse and high-dimensional systems---controllability fails. In these regimes, it remains unclear whether an experiment can still uniquely determine the full system or parts of the system. For \textit{autonomous} systems, i.e., systems without a control input, \citet{stanhope2014identifiability} and \citet{qiu2022identifiability} show that identifiability from a single trajectory reduces to whether the orbit is confined to a proper subspace that is closed under the action of the dynamics. Extending this picture to \textit{controlled} systems requires accommodating control aspects and disentangling the roles of the initial state and the input. To our knowledge, no closed-form characterization of identifiability exists for non-autonomous LTI systems.

We bridge this gap by deriving an experiment-conditional identifiability characterization for controlled LTI systems that does not assume controllability. We derive the \emph{visible subspace}, the part of the state space that the experiment can reach. We show that systems consistent with an experiment must induce the same dynamics on the visible subspace, while the dynamics can differ arbitrarily outside of it.
This leads to our main contributions: (i)~a partial identifiability theorem characterizing the subsystem uniquely determined by a given experiment,
(ii)~an explicit, closed-form parametrization of all systems consistent with the experiment, and
(iii)~ algebraic criteria that quantify proximity to non-identifiability. Finally, we validate these results empirically, demonstrating their relevance for sparse systems.

\begin{figure}[ht]
    \centering
    \includegraphics[width=0.45\textwidth]{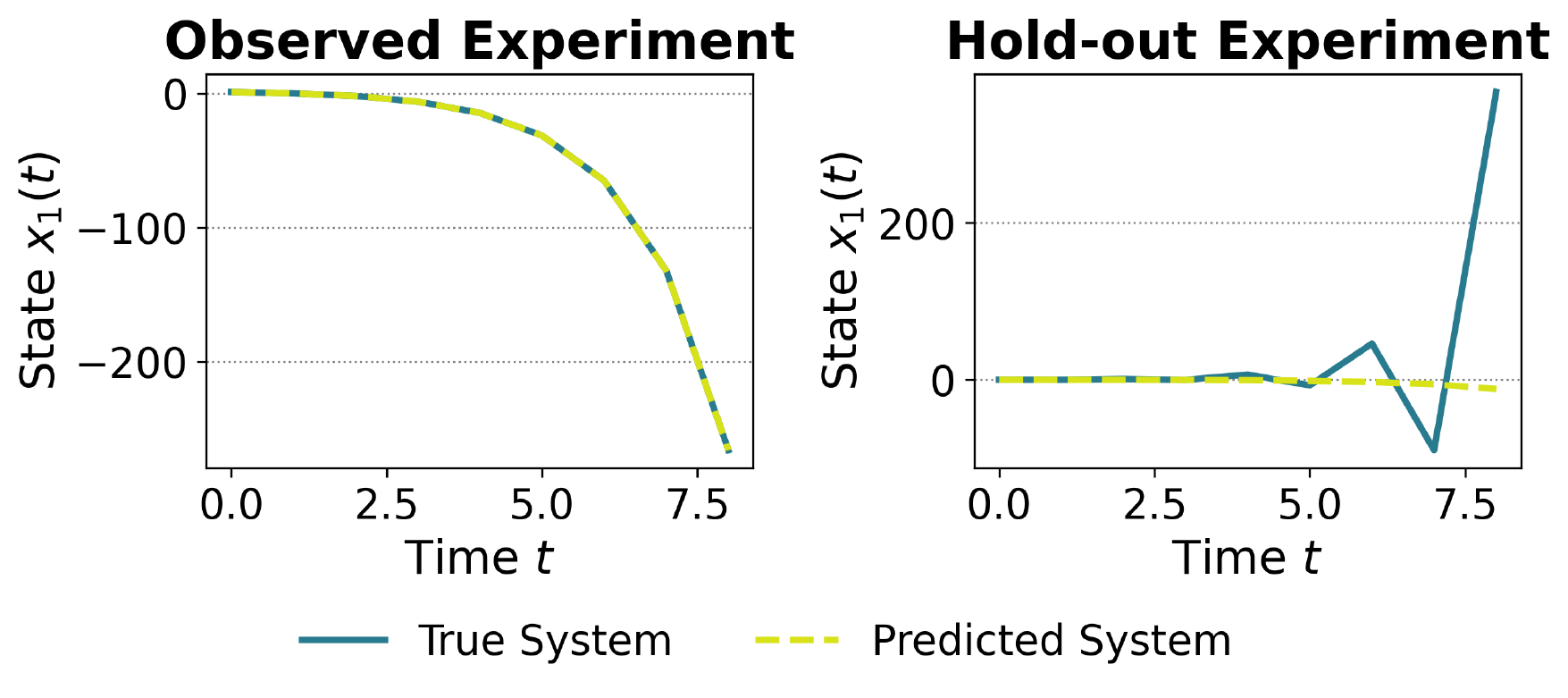}
    \caption{\textbf{Non-identifiability from a single realized experiment.} Two distinct LTI-systems are indistinguishable from the observed experiment (left), but they yield different trajectories under a new experimental condition (right).}
    \label{fig:non-identifiability}
\end{figure}

\section{Problem Statement}\label{sec:preliminaries}

\textbf{Problem Setting.} We consider the inverse problem of identifying the parameters of an unknown, continuous-time linear time-invariant (LTI) system from observed state-input trajectories. The system dynamics are governed by a linear ordinary differential equation (ODE)
\begin{align}
\label{eq:LTI}
\dot{\rx}(t) = A\rx(t) + B\ru(t), \;\; 
\rx(0) = \rx_0 \in \mathbb{R}^{n},
\end{align}
where $\rx(t) \in \brn{n}$ is the state, $\ru(t) \in \brn{m}$ is the control input, and the matrices $(A, B) \in \mathbb{R}^{n \times n} \times \mathbb{R}^{n \times m}$ are the unknown system parameters to be recovered \footnote{We assume exact full-state measurements; however, our theoretical analysis extends to systems with known full-rank linear observation maps $C$ via the exact reconstruction $\rx(t)=C^\dagger\ry(t)$. However, in practice, ill-conditioned $C$ may introduce numerical challenges under noisy observations.}.

We distinguish between an \emph{experiment class} $\calE=(\rx_0, \calU)$, defined by a fixed initial condition and a set of admissible control inputs $\calU \subseteq L^2([0,T], \brn{m})$\footnote{$L^2([0,T], \brn{m})$ is the space of square-integrable $\brn{m}$-valued functions on $[0,T]$.}, and a \emph{single experiment} (or realization) $\re=(\rx_0, \ru)$ for a specific $\ru \in \calU$. An experiment induces an observed state trajectory, denoted by $\varphi([0,T] \mid A,B,\rx_0,\ru)$, where the solution at time $t \in [0,T]$ is given by:
\begin{align*}
\varphi(t \mid A,B,\rx_0,\ru)
= e^{At}\rx_0 + \int_{0}^{t} e^{A(t-s)}B\ru(s)\, ds. 
\end{align*}

Let $\Omega = \brn{n \times n} \times \brn{n \times m}$ denote the parameter space. We define the \emph{experiment-consistent} set as 
\begin{align}
    \label{eq:class-x0-u}
    \notag
    [A, B]_{(\rx_0, \ru)}:= \big\{\abtilde \in \Omega:\tilde{\varphi}([0,T]\spbar \abtildenobracket, \rx_0, \ru) \\
    =\varphi([0,T]\spbar A, B, \rx_0, \ru)\big\}.
\end{align}

This set describes all systems that induce the same observed trajectory and thereby represents the immediate uncertainty facing a practitioner after a single experiment. We say that $(A,B)$ is \emph{identifiable from $\re$} if and only if this set is a singleton,
\begin{align}
\label{eq:identifiability_singleton}
[A,B]_{\re}=\{(A,B)\}.
\end{align}

To analyze fundamental limits of identifiability independently of degenerate realizations (e.g., $\ru=\mathbf{0}$), we also define the \emph{design-consistent} set:
\begin{align}
    \label{eq:class-x0-calU}
    \notag
    [A, B]_\calE&:= \big\{\abtilde \in \Omega: \forall \ru \in \calU, \\
    \notag
    & \varphi([0,T]\spbar \abtildenobracket, \rx_0, \ru) 
    =\varphi([0,T]\spbar A, B, \rx_0, \ru)\big\} \\
    &=\underset{\ru \in \calU}{\bigcap} [A, B]_{(\rx_0, \ru)}.
\end{align}
This set represents the theoretical limit of information extractable given the experiment class $\mathcal{E}=(\rx_0, \mathcal{U})$. These systems cannot be distinguished for any choice of control input $\ru \in \mathcal{U}$. Note that, for any realized experiment $\re \in \calE$, we have the following nesting relationship (see also \Cref{fig:PEness-collapse}) :
\begin{align}
    \underbrace{\{\abpar\}}_\text{Truth} \subseteq \underbrace{[A, B]_\calE}_{\substack{\text{Design-} \\ \text{ consistent Set}}} \subseteq \underbrace{[A, B]_\re}_{\substack{\text{Experiment-} \\ \text{ consistent Set}}} \subseteq \underbrace{\Omega}_\text{Prior}. \label{eq:set_hierarchy}
\end{align}

In practice, if any prior structural information about $\abpar$ is available, we can further restrict the prior parameter space $\Omega \subseteq \brn{n \times n} \times \brn{n \times m}$, assuming $\abpar \in \Omega$. Our analysis naturally extends to these settings by intersecting the design- and experiment consistent sets in \Cref{eq:set_hierarchy} with the restricted prior $\Omega$.

\textbf{Objective.} Our goal is to (i)~ characterize the conditions under which a system can be uniquely identified from a single experiment $(\rx_0, \ru)$ (i.e., $[A,B]_{\re} = \{\abpar\}$), and (ii)~ characterize the subsystem that can be uniquely identified in the general case where the full system $\abpar$ is not necessarily identifiable. To do so, we separate the geometric limits imposed by $\rx_0$ from the excitation properties of $\ru$, and derive algebraic conditions under which $[A, B]_\re$ collapses to the structural limit $[A,B]_\calE$, and ultimately to the ground truth $\{(A,B)\}$.

To formalize excitation properties, we adopt the continuous-time notion of persistent excitation provided by \cite{rapisarda2022persistency}.

\begin{definition}[Persistent Excitation]
\label{def:PEness}
An at least $r-1$ times continuously differentiable signal $\ru:[0,T] \to \brn{m}$ (with $T \in \mathbb{R}_+$) is persistently exciting (PE) of order $r$ if, for every $\eta^\top \in \brn{rm}$ it holds that
\begin{align}
\label{eq:PEness-cont}
     \left(\forall t \in [0,T]:\eta^\top \begin{bmatrix}
         \ru(t) \\ \ru^{(1)}(t) \\ \vdots \\ \ru^{(r-1)}(t)
     \end{bmatrix}=\mathbf{0} \right) \implies 
     \left(\eta =\mathbf{0}\right)
\end{align}
where $\ru^{(j)}$ denotes the $j-$th derivative \footnote{For the discrete-time setting, the analogue is that the input Hankel matrix of order $r$ has full rank. The formal correspondence is established in \Cref{proof:ZOH-invariance}, Invariants under ZOH Sampling.}. 
\end{definition}

In addition, we require a non-pathological design assumption ensuring that the equality over responses over the admissible class $\calU$ extends to all $L^2$ control inputs, allowing us to analyze the structural identifiability limits of the system.

\begin{assum}[Richness of the Input Class]
\label{assum:input-class-richness}
    The admissible input class $\mathcal{U} \subseteq L^2([0,T], \brn{m})$ is dense in  $L^2([0,T],\brn{m})$, contains the zero signal $u \equiv 0$, and contains at least one persistently exciting input (\Cref{def:PEness}).
\end{assum}

\textbf{Extension to Multiple Experiments.} If instead of a single experiment $\re$, we have access to $N$ experiments $E:=\{\re_i\}_{i=1}^N$ (possibly with different initial conditions and/or admissible input classes) from the same underlying system $(A,B)$, then the set of systems consistent with the set of observations $E$ is given by $[A,B]_E \;=\; \bigcap_{i=1}^{N} [A,B]_{\re_i}$.
Here, $(A,B)$ is identifiable from $E$ if and only if $[A,B]_E$ is a singleton. This intersection formalizes the intuition that the uncertainty class shrinks strictly when additional experiments probe previously unexcited directions of the state space.

\textbf{Extension to Discrete-Time.} While we present the analysis in continuous time, our framework is algebraic and translates seamlessly to discrete-time settings. For piecewise-continuous inputs applied via zero-order-hold (ZOH), the exact discrete-time model $(A_d, B_d) = (e^{A\Delta t},\; \int_0^{\Delta t} e^{As}B\spone ds)$ preserves our identifiability results under standard nonpathological sampling. Furthermore, the continuous-time persistency of excitation requirement maps directly to the discrete condition that the input Hankel matrix of order $k+1$ has full row rank (see \cref{proof:ZOH-invariance}).

\section{Theoretical Analysis}\label{sec:ta}

This section provides an \emph{experiment-conditional} characterization of identifiability from a single experiment $\re=(\rx_0, \ru)$. We proceed in four logical steps: 
First, we introduce the \emph{visible subspace} $V(\rx_0)$-- the geometric limit of what an experiment class $\mathcal{E}$ can probe.
Second, we link the experiment class to a single experiment and demonstrate that a single persistently exciting control input maximizes information content in a single trajectory, forcing the experiment-consistent set to collapse to the design-consistent set, i.e., $[A,B]_\re = [A,B]_{\mathcal{E}}$. Third, we develop a computable, closed-form parameterization of $[A, B]_\re$ that characterizes the identifiable subsystem for an informative experiment $\re$. Finally, we propose two explicit, algebraic criteria that quantitatively describe a system's proximity to structural non-identifiability for an informative experiment.

\textbf{The Visible Subspace.} Excitation enters the system through two primary channels: the ``initial excitation'' injected by the state $\mathbf{x}_0$ and the control inputs $u \in \mathcal{U}$ filtered through the map $B$. The excitation propagates through the state space via the repeated action of the transition matrix $A$, defining a structural ``information ceiling''---the smallest $A-$invariant subspace that contains all potential sources of excitation.

\begin{definition}[Visible subspace]
\label{def:visible-subspace}
For any system $(A,B)\in\Omega$ and initial state $\rx_0 \in \brn{n}$, define the \emph{visible subspace} as
\begin{align}
\label{eq:visible-subspace}
V(\rx_0)
&\;\coloneqq\;
\mathcal{K}_n\!\big(A,\,[\rx_0 \;\; B]\big) \\
\notag
&\;=\;
\operatorname{span}\Big\{[\rx_0 \;\; B],\,A[\rx_0 \;\; B],\,\dots,\,A^{n-1}[\rx_0 \;\; B]\Big\},
\end{align}
where $\mathcal{K}_n$ denotes the $n-$th order Krylov span generated by $\Kcore \in \brn{n \times (m+1)}$, the column-wise-stacked initial state vector and control matrix. Equivalently, $V(\rx_0)$ is the smallest $A$-invariant (i.e., $AV(\rx_0)\subseteq V(\rx_0)$) subspace containing the initial state and the input image, $\{\rx_0\}\cup\mathrm{Im}(B)$. (The formal proof is provided in \cref{proof:minimality-of-V(E)}).
\end{definition}

The visible subspace imposes a strict boundary: $V(\rx_0)$ contains all directions of excitation and is invariant under $A$, therefore the trajectory induced by any experiment $\re$ is mathematically locked within the visible subspace. Regardless of the input's complexity or richness, the system cannot develop components in the ``blind spot'' outside $V(\rx_0)$. We formalize this restriction as \emph{orbit confinement}:

\begin{lemma}[Orbit confinement]
\label{lemma:orbit-confinement}
For any system $(A,B)\in\Omega$, initial state $\rx_0 \in \brn{n}$, and admissible input $\ru \in \mathcal{U}$, the resulting state trajectory is strictly confined to the visible subspace: $\varphi(t|A, B, \exptuple)\in V(\rx_0)$ for all $t \geq 0$.
\end{lemma}

\begin{figure}[ht]
    \centering
    \hspace{-0.7cm}
    \includegraphics[width=0.8\linewidth] 
    {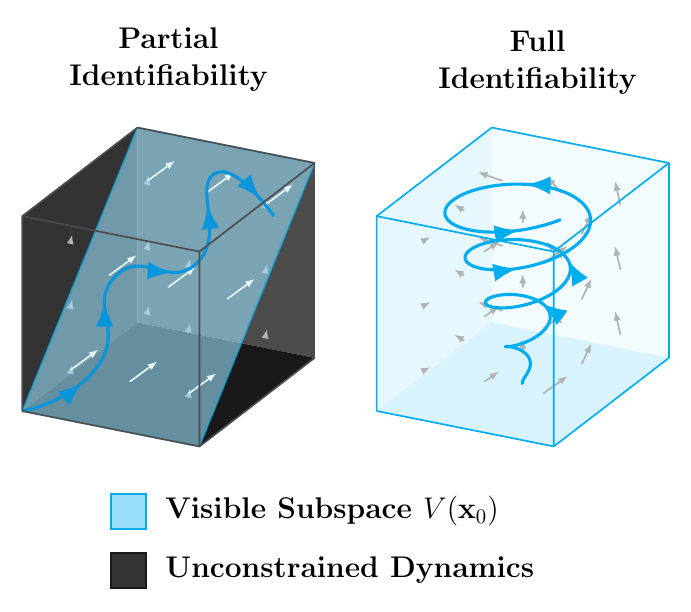}
    \caption{\textbf{Partial vs. Full Identifiability of a Three-dimensional System.} (Left) The realized trajectory stays in a proper visible subspace (blue), so only the dynamics induced on that subspace are identifiable (i.e., this restriction represents the identifiable subsystem); dynamics on the complement (gray) are unconstrained by the experiment. (Right) The visible subspace spans the entire state space, and thus the full system is identifiable.}
    \label{fig:teaser-pic}
\end{figure}
 
Orbit confinement allows us to separate the state space into the \emph{visible} directions probed by the experiment, which span $V(\rx_0)$, and an arbitrary complementary subspace $W$ of unexcited ``blind'' directions. Using a basis adapted to the direct sum $V(\rx_0)\oplus W=\brn{n}$, we can explicitly isolate the dynamics on $\vissub$-- which, as we will show, are uniquely determined by the experiment-- from the unconstrained dynamics on $W$.

\begin{lemma}[Block form in a $(V,W)$ basis]
\label{lemma:block-form-in-VW}
For any system $(A,B)\in\Omega$ and initial state $\rx_0 \in \brn{n}$, let $V(\rx_0)$ be as in \Cref{eq:visible-subspace} and $V\coloneqq V(\rx_0)$ with $k\coloneqq\dim(V)\le n$. Choose any complement $W$ such that $V\oplus W=\mathbb{R}^n$, and pick a change of coordinates $T=[P\;\;Q]$ with $\mathrm{Im}(P)=V$ and $\mathrm{Im}(Q)=W$. Then:
\begin{align}
\label{eq:block-form}
\notag
T^{-1}AT&=\begin{bmatrix}
A_V & A_*\\[2pt]
0 & A_W
\end{bmatrix}, \\
\;\;
T^{-1}B=\begin{bmatrix}
B_V\\[2pt]
0
\end{bmatrix}&,
\;\;
T^{-1}\rx_0=\begin{bmatrix}
\rx_{0,V}\\[2pt]
0
\end{bmatrix}.
\end{align}
In particular, any realized trajectory under $\ru \in \mathcal{U}$, $\{\rx(t)\}_{t\in[0,T]}$ depends only on $(A_V,B_V,\rx_{0,V})$ and is independent of $A_*$ and $A_W$. 
\end{lemma}

This lemma yields the structural consequences of orbit confinement (\cref{lemma:orbit-confinement}) by showing that only the reduced dynamics on $\vissub$ can affect the observed trajectory. 
In coordinates adapted to $\brn{n}=V\oplus W$, the $A-$invariance of $V$ implies that $T^{-1}AT$ is block upper-triangular, while $\mathrm{Im}(B) \subseteq V$ and $\rx_0 \in V$ imply that both $T^{-1}B$ and $T^{-1}\rx_0$ have zero $W-$components. Hence, if $z(t)=T^{-1}\rx(t)$ and $z(t)=(z_V(t), z_W(t))$ then $z_W(t)$ satisfies an autonomous linear subsystem with zero initial state, and therefore $z_W(t)\equiv 0$ for all $t$. It follows that the realized trajectory depends only on the reduced triple $(A_V , B_V , \rx_{0,V})$ and is independent of the blocks $A_W$ and $A_*$. Equivalently, the visible subspace $V(\rx_0)$ determines the design-consistent set $[A,B]_{\mathcal E}$: no experiment generated from $(\rx_0,\mathcal U)$ can identify dynamics acting outside of $V(\rx_0)$. Any mechanistic claim must therefore be interpreted relative to the visible subspace induced by the experiment design.

\textbf{From Design Limits to Experiment Limits.}
The results above are \emph{structural}: they characterize the maximal information that is, in principle, recoverable from the experiment design $\mathcal{E}$. In practice, however, the practitioner often conducts only a single experiment $\re=(\rx_0,\ru)$ and thus faces the larger experiment-consistent set $[A,B]_{\re}$. In the following, we analyze under which conditions this larger set collapses to the structural limit $[A,B]_{\mathcal{E}}$. In particular, we show that this depends on the \emph{information content of the observed trajectory}: even when $(A,B)$ is structurally identifiable within $V(\rx_0)$, an uninformative control input may fail to excite all in-principle identifiable directions. This motivates characterizing the informativeness of an observed trajectory with respect to the visible subspace. 

\begin{figure}[ht]
    \centering
    \includegraphics[width=0.75\linewidth]{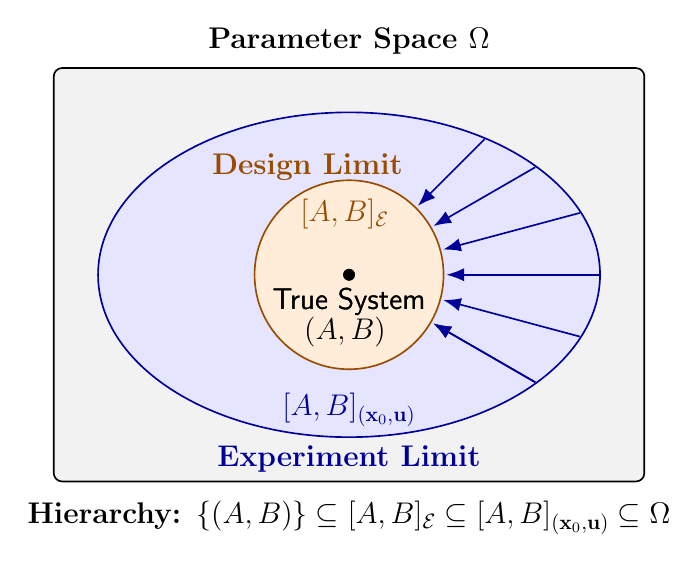}
    \caption{\textbf{Hierarchy of system uncertainty sets for a fixed initial state $\rx_0$.} Experiment-consistent set $[A,B]_\re$
    can be larger than the design-consistent set $[A,B]_{\mathcal E}$, but collapses to it under trajectory informativeness.}
    \label{fig:PEness-collapse}
\end{figure}

\begin{definition}[Informative Trajectory]
\label{def:informative_trajectory}
Let $k \coloneqq \dim(V(\rx_0))$. Let $P_V \in \mathbb{R}^{n \times k}$ be a matrix whose columns form an orthonormal basis for $V(\rx_0)$. We define the minimal coordinates $\xi(t) \in \mathbb{R}^k$ via the projection $\xi(t) \coloneqq P_V^\top \rx(t)$. Define the joint regressor $z(t) \coloneqq [\xi(t)^\top, \ru(t)^\top]^\top \in \mathbb{R}^{k+m}$.
A realized trajectory is \emph{informative} if the finite-time Gramian of the joint regressor $z(t)$ has full rank:
\[
\begin{aligned}
    \mathcal{G}_{[0,T]} &\coloneqq \int_0^T z(t)z(t)^\top dt \\
&= \int_0^T \begin{bmatrix} \xi(t) \ \ru(t) \end{bmatrix} \begin{bmatrix} \xi(t) \ \ru(t) \end{bmatrix}^\top dt \succ 0.
\end{aligned}
\]
\end{definition}
\cref{def:informative_trajectory} delivers an experiment-level condition that can be verified post-hoc from the observed trajectory: it asks whether the realized joint regressor spans all identifiable directions. The next lemma complements this with a pre-experiment design criterion: persistent excitation \citep{shimkin1987persistency, rapisarda2022persistency} of the input of sufficient order guarantees trajectory informativeness without requiring knowledge of the true $\abpar$.

\begin{lemma}[PE Transfer]
\label{lemma:pe_transfer}
Consider the restricted dynamics on the visible subspace $V(\rx_0)$ of dimension $k$:
\begin{align}
    \dot{\xi}(t) = A_V \xi(t) + B_V \ru(t), \;\; \xi(0)=\rx_{0,V},
\end{align}
where $A_V \coloneqq P_V^\top A P_V$, $B_V \coloneqq P_V^\top B$ and $\rx_{0,V}\coloneqq P_V^\top \rx_0$. Define the joint regressor $z(t):=\begin{bmatrix}
    \xi(t)^\top, \; \ru(t)^\top
\end{bmatrix}^\top$, and its finite-time Gramian $\mathcal{G}_{[0,T]}=\int_0^T z(t)z(t)^\top dt$. Then, if the input $\ru: [0,T] \to \brn{m}$ is PE of order at least $k+1$ in the sense of \cref{def:PEness}, the realized trajectory is informative in the sense of \cref{def:informative_trajectory}, i.e., $\mathcal{G}_{[0,T]}\succ 0$.
\end{lemma}

Next, we formalize the link between the experiment-consistent set and the design-consistent set. Intuitively, if a single experiment is sufficiently informative, it exhaustively sweeps the visible subspace, leaving us with the same uncertainty as if we had performed all possible experiments $\{(\rx_0, \ru):\ru \in \calU\}$.

\begin{theorem}[The Bridge]
\label{theorem:bridge}
Let $(A,B)\in\Omega$ and let $\mathcal{E}=(\rx_0,\calU)$ be an experiment design. If the realized trajectory generated by $\ru \in \mathcal{U}$ is informative (Def. \ref{def:informative_trajectory}), then the single-trajectory consistency set collapses to the design-level consistency set: $[A, B]_{\re} = [A, B]_{\mathcal{E}}$.
\end{theorem}

\Cref{theorem:bridge} guarantees that a single informative experiment suffices to reach the design-level information ceiling.

\textbf{Experiment-Conditional Identifiability.} 
Next, we characterize the structure of the experiment-consistent set under an informative experiment by combining \Cref{theorem:bridge} and \Cref{lemma:block-form-in-VW}. Since \cref{theorem:bridge} guarantees $[A, B]_{\re} = [A, B]_{\mathcal{E}}$ under informativeness, this characterization simultaneously describes the design-consistent set: a single informative experiment achieves the full information ceiling imposed by the experiment class $\mathcal{E}$. Specifically, all experiment-consistent systems must agree on the dynamics restricted to the visible subspace $V(\rx_0)$, while they may differ arbitrarily on the complement $W$. Accordingly, full identifiability holds precisely when the visible subspace spans the entire state space (i.e., $V(\rx_0)=\brn{n}$); otherwise, identifiability is partial and only the visible subsystem, i.e., the dynamics restricted to the visible subspace, is identified by the experiment. \Cref{fig:teaser-pic} illustrates the difference between partial and full identifiability.

\begin{theorem}[Experiment-conditional identifiability]
    \label{theorem:identifiability}
    Let $\abpar \in \Omega$ and let $\re=(\exptuple)$ be a realized experiment satisfying \cref{def:informative_trajectory}. Define the visible subspace $V(\rx_0)$ as in \cref{def:visible-subspace}, and let $k\coloneqq\dim{V(\rx_0)}\leq n$. Choose any complement $W$ such that $V(\rx_0)\oplus W =\brn{n}$, and let $T=[P \;\;Q]$ be any basis adapted to $V(\rx_0)\oplus W=\brn{n}$.
    \begin{enumerate}[label=\textbf{(\roman*)},leftmargin=*]
    \item \textbf{Identifiable subsystem.} For any $\abtilde \in [A,B]_\re$, we necessarily have
    \begin{align*}\tilde{A}\big|_{V(\rx_0)}=A\big|_{V(\rx_0)} \;\; \text{and} \;\; \tilde{B}=B.
    \end{align*}
    Conversely, if $\tilde{A}|_{V(\rx_0)}=A|_{V(\rx_0)}$ and $\tilde{B}=B$, then $\abtilde \in [A,B]_\re$. Hence, the restriction $(A|_{V(\rx_0)}, B)$ is uniquely determined by the single realized experiment $\re$, even when $\abpar$ as a whole is not identifiable.
    \item \textbf{Characterization of the experiment-consistent set $[A,B]_\re$.} The matrices $\abtilde \in [A,B]_\re$ can be characterized as:
    \begin{align}
    \label{eq:equivalence-class}
        \notag
        &[A, B]_\re=\Big\{\abtilde \in \Omega: \inv{T}\tilde{A}T=\begin{bmatrix}
            A_V & \Theta \\
            0 & \Psi
        \end{bmatrix}, \;\; \\
        &\qquad \tilde{B}=B, \;\;\begin{bmatrix} \Theta \\ \Psi \end{bmatrix} \in \brn{n \times (n-k)}  \Big\},
    \end{align}
    with $n(n-k)$ degrees of freedom.
    \item \textbf{Full identifiability.} The system $(A,B)$ is uniquely identifiable from $\re$ (i.e. $[A, B]_\re=\{\abpar\}$) if and only if $V(\rx_0)=\brn{n}$. Equivalently, the $A-$cyclic Krylov sequence of order $n$ generated by $\Kcore$ has rank $n$.
    \end{enumerate}
\end{theorem}
\textit{We provide the formal proof in \Cref{proof:identifiability-1}.}

By \Cref{theorem:bridge}, the experiment-consistent set equals the design-consistent set. Therefore, the characterization in \Cref{eq:equivalence-class} also describes $[A, B]_\calE$, and the criterion for full identifiability from a single experiment (\Cref{theorem:identifiability} (iii)) is equivalently a criterion for full identifiability from the design class $\calE$.

\textbf{Tests for Full Identifiability.} To test if a system is fully identifiable from a single experiment $\re$, we need to evaluate two conditions. First, we need to test if the induced experimental trajectory is informative (informative trajectory requirement). Trajectory informativeness can be ensured by using a persistently exciting input $\ru$ of order at least $k+1$ (\Cref{def:PEness}). Second, we need to test if the visible subspace spans the full state space, i.e., $V(\rx_0) = \brn{n}$ (full-visibility condition). We propose two algebraic criteria to test the latter condition.
These criteria build on a controllability test of the augmented system pair $(A, \Kcore)$. This is motivated by the fact that $V(\rx_0)$ is precisely the Kalman reachable subspace of the augmented pair $(A,\Kcore)$.

\begin{enumerate}[label=\textbf{(\roman*)},leftmargin=*]
    \item \textbf{(Left-eigenvector test).} Full identifiability fails if and only if there exists a left eigenpair $(\lambda, w)$ of $A$, with $w^\top A=\lambda w^\top$, such that $w^\top \rx_0=0$ and $w^\top B=\mathbf{0}$ (equivalently, $w^\top\Kcore=\mathbf{0}$). To quantify the proximity to this failure, for each left eigenpair $(\lambda_i, w_i)$, we define the normalized alignment:
    \begin{align*}
    \mu_i(A,B\mid \rx_0)&\coloneqq \frac{\|w_i^\top [\rx_0 \;\; B]\|_2}{\|w_i\|_2 \|\Kcore\|_2},
        \;\; \\
        \mu_{\min}&\coloneqq \min_i \mu_i \in [0,1].
    \end{align*}
    Under a sufficiently rich control input, the system is fully identifiable if and only if $\mu_{\min}>0$.
    \item \textbf{(Fixed-experiment PBH test).} The classical Popov--Belevitch--Hautus (PBH) criterion measures distance to uncontrollability via the matrix pencil $[\lambda I-A, B]$. Because $\rx_0$ is a fixed property of the executed experiment and not a system parameter subject to uncertainty, any unobservable mode must be strictly orthogonal to the exact nominal $\rx_0$. 
    Define $\rx_0^\perp \coloneqq \{w \in \brn{n}: w^\top\rx_0=0\}$ and let $Q\in\mathbb R^{n\times(n-1)}$ have orthonormal columns spanning $\rx_0^\perp$ (take $Q=I_n$ if $\rx_0=\mathbf{0}$). Let $\mathcal H(\lambda)\coloneqq[\lambda I-A,\ B]$. We define the fixed-experiment PBH margin as:
    \begin{align}
    \label{eq:pbh-test}
        d_{\mathrm{PBH}}(A,B\mid \rx_0) \coloneqq \inf_{\lambda\in \sigma(A)} \sigma_{\min}\big(Q^\top \mathcal H(\lambda)\big),
    \end{align}
    where $\sigma_\text{min}(\cdot)$ denotes the smallest singular value. Under a sufficiently rich control input, the system is fully identifiable if and only if $d_{\mathrm{PBH}}(A,B\mid \rx_0)>0$.
\end{enumerate}

Both quantities, $\mu_i(A,B\mid \rx_0)$ and $d_{\mathrm{PBH}}(A,B\mid \rx_0)$, can also be interpreted as \emph{margins} to non-identifiability (provided that the informative trajectory requirement is satisfied). Specifically, when the margins are very small, the inverse problem might be numerically ill-conditioned: an estimator that inverts a near-rank-deficient regressor might struggle with identifying the true system, even though the system is theoretically fully identifiable.

\section{Experiments}\label{sec:experiments}

We empirically validate our theoretical results on synthetic LTI systems. We focus on sparse systems as they are a natural testbed for studying identifiability failures \citep{casolo2025identifiabilitychallengessparselinear}.
Our goals are twofold: (i)~to show that sparse systems are commonly uncontrollable, making classical identifiability guarantees inapplicable, and (ii)~to validate \Cref{theorem:identifiability} 
by testing whether, even when the full system is not identifiable, the restriction of $\abpar$ to the visible subspace $V(\rx_0)$ is uniquely determined. To this end, we generate sparse-continuous system ensembles of $(A, B, \rx_0)$, simulate trajectories under sufficiently exciting controls, and benchmark four representative identification methods: Dynamic Mode Decomposition (DMDc) \citep{doi:10.1137/15M1013857}, Multivariable Output-Error State-Space (MOESP) \citep{6530030}, SINDy \citep{brunton2016discovering}, and neural ODEs (NODEs) \citep{chen2018neural} \footnote{In our setting, this is more precisely a gradient-trained linear ODE. We retain the NODE label to situate it within the \citet{chen2018neural} framework from which the training procedure is adopted.}. We compare full-system and visible-subsystem recovery across each.

\subsection{Relevance of Identifiability for Sparse Systems}
\label{ssection:sparse-systems}
We evaluate how sparsity impacts controllability and thereby standard identifiability guarantees. Moreover, we investigate whether identifiability can still hold for selected initial states, even when the system is not controllable.

\textbf{Experimental Setup.} We first generate a base ensemble of system matrices $\abpar$ using the Ginibre ensemble with sparsification:
\begin{align*}
    A_{ij}&=X_{ij}M_{ij}, \;\; X_{ij} \sim \frac{1}{\sqrt{n}}\mathcal{N}(0, 1), \;\; M_{ij}\sim \operatorname{Ber}(p),
\end{align*}
where $(1-p)$ is the target sparsity level, and analogously for $B$. We generate systems with varying dimensionality and sparsity levels and simulate $1000$ systems per dimensionality-sparsity combination. We evaluate whether a simulated system $\abpar$ is controllable by checking the rank of the Kalman controllability matrix \citep{wonham1974linear,sontag2013mathematical}, i.e. $\mathbf{1}\{\operatorname{rank}(\mathcal{C}_n(A, B))<n\}$.

To test if uncontrollable systems can still be identifiable for specific initial states, we use systems with intermediate densities $0.3 \leq p(A, B)\leq 0.7$, which avoids trivial extremes and yields a mixture of controllable and uncontrollable systems. We generate different types of initial conditions by varying their sparsity level. Concretely, for each sparsity level $p_{\rx_0}\in \{0.25, 0.5, 0.75, 1\}$, we generate an ensemble of initial states using the following procedure: 
\begin{enumerate}[leftmargin=*]
    \item Sample $\rx_0$ uniformly from the unit sphere $\mathbb{S}^{n-1}$,
    \item Apply a Bernoulli sparsification mask, keeping each entry independently with probability $p_{\rx_0}$, and
    \item Scale the initial states to have unit norm for $p_{\rx_0}<1$. 
\end{enumerate}
We sample $100$ initial states per $p_{\rx_0-}$value, pair each density-filtered system $\abpar$ with the resulting $\rx_0-$ensemble to form triples $(A, B, \rx_0)$, and compute the PBH identifiability criterion $d(A, B \spbar \rx_0)$. We label a triple as identifiable if $d(A, B \spbar \rx_0)>10^{-6}$.
Additional details on the experimental setup can be found in \cref{ssection:experimental-setup}. 

\begin{figure}[ht]
    \centering
    \includegraphics[width=0.315\textwidth]{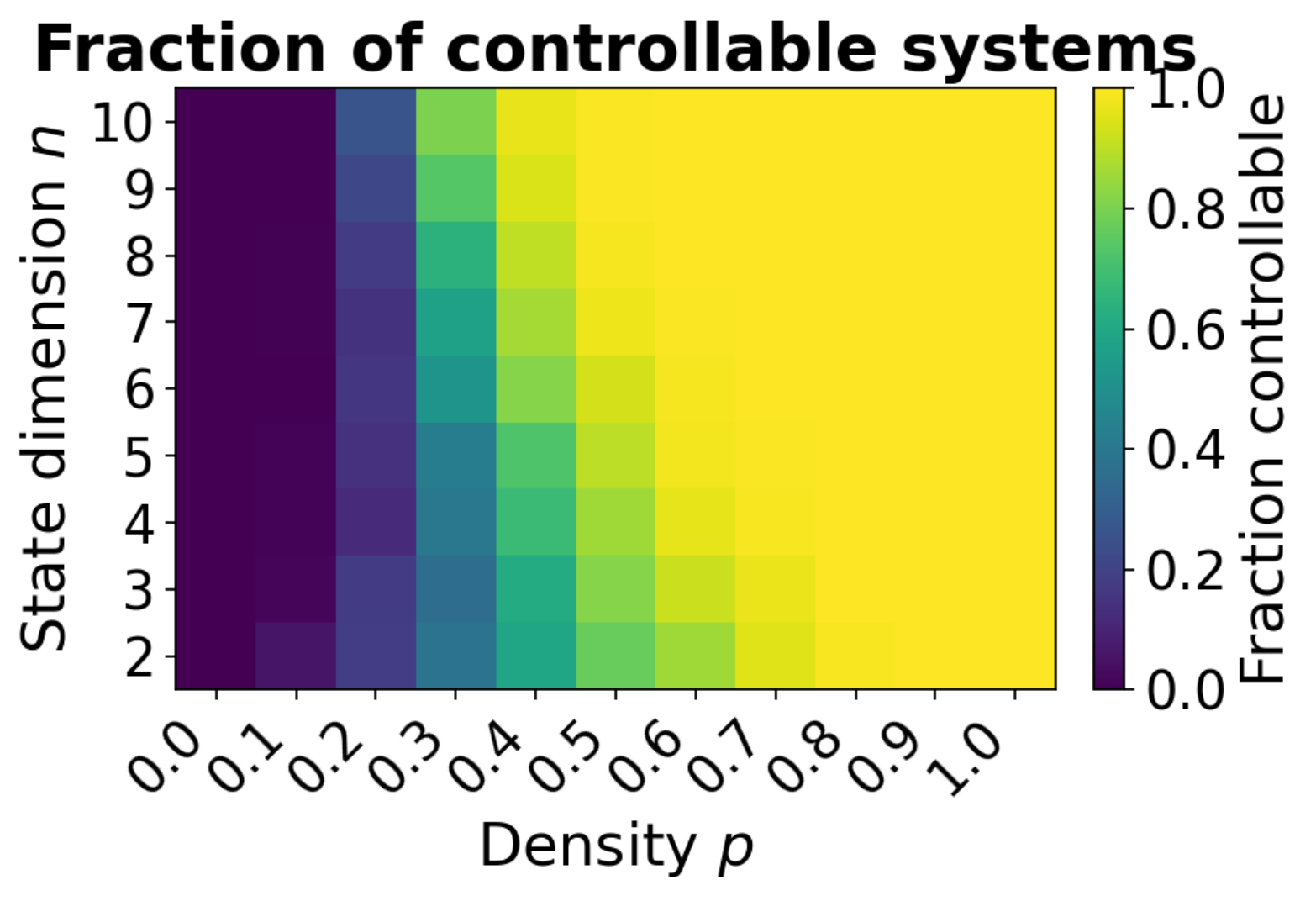}
    \caption{\textbf{Identifiability in sparse LTI systems.} Fraction of controllable systems $\abpar$ under varying densities ($\rx$-axis) and state dimensions ($\ry$-axis).}
    \label{fig:controllability-heatmap}
\end{figure}

\begin{table}[H]
    \centering
    \vspace{-0.2cm}
    \begin{tabular}{lcccc}
        \toprule
        $p_{\rx_0}$ & 0.25 & 0.50 & 0.75 & 1.0 \\
        \midrule
        $\tau_{\rx_0}$ & 0.27 & 0.52 & 0.76 & 1.00 \\
        \bottomrule
    \end{tabular}
    \caption{\textbf{Identifiability vs. Initial State Density.} Fraction of identifiable triples ($\tau_{\rx_0}=\frac{1}{N}\underset{i=1}{\overset{N}{\sum}}\mathbf{1}\{d_\text{PBH}(A_i, B_i \spbar \rx_0)>10^{-6}\}$) as $x_0$ density increases.}
    \label{tab:x0-density}
\end{table}

\textbf{Results.} \Cref{fig:controllability-heatmap} quantifies how often random systems are controllable as a function of sparsity and dimensionality. As sparsity increases, controllability becomes much less frequent. This is consistent with the intuition that fewer state couplings and fewer effective input pathways limit the reachability of the full state space. In this regime, classical identifiability guarantees, which rely on controllability, no longer apply; our analysis determines what a specific experiment can identify.

\Cref{tab:x0-density} shows that identifiability in the uncontrollable regime depends strongly on the initial condition (provided that there exists a persistently exciting control input). Dense initial conditions, i.e., higher values for $p$, tend to excite a larger visible subspace and therefore more often lead to full identifiability from a single experiment. In contrast, sparse initial conditions (small $p$) shrink $V(\rx_0)$, leading to non-identifiability. These results confirm that identifiability is experiment-conditional.

\subsection{Subsystem Recovery}
\label{ssection:subsystem-recovery}

We evaluate whether state-of-the-art system identification methods recover the identifiable subsystem given in \cref{theorem:identifiability}.
Concretely, we compare full-parameter recovery to recovery within $V(\rx_0)$ across varying observation noise and visibility levels $k=\operatorname{dim}(V(\rx_0))$.

\textbf{Experimental Setup.} We generate an ensemble of LTI systems $\abpar$. As in \cref{ssection:sparse-systems}, each matrix is sampled from a sparse-continuous distribution with density $p=0.1$ and state dimension $n=10$. The nonzero system parameters are sampled entrywise i.i.d. from a truncated Gaussian, i.e., $\mathcal{N}(0,1)$ truncated to $(-\infty, -0.1] \cup [0.1, \infty)$.  We stabilize $A$, so that the spectral radius satisfies $\rho(A)\leq 0.95$ and curate uncontrollable systems to focus on regimes where identification is challenging. We sample the initial state $\rx_0-$ensembles with different density levels, as in \Cref{ssection:sparse-systems}. For each $(A, B, \rx_0)$, we compute the visible subspace basis $P \in \brn{n \times k}$ and the number of ``visible dimensions'' $k$ using the Krylov sequence in \cref{def:visible-subspace} and SVD thresholding. We stratify triples by their visibility rank $k$, resulting in a total of $270$ triples, with $45$ samples per ``visible dimension'' $k \in \{5,6,7,8,9,10\}$.

For each triple, we generate a trajectory of length $T=80$ by simulating an experiment with a sufficiently exciting input. In the noise-robustness experiment (\cref{fig:fullsys-vs-vissys}, top row), we add i.i.d.\ Gaussian noise to the state trajectory with standard deviation scaled by $\sigma \in \{0, 10^{-3}, 10^{-2}, 10^{-1}, 0.5\}$. We use four system identification methods to obtain estimates $(\hat{A}, \hat{B})$: DMDc \citep{doi:10.1137/15M1013857}, MOESP \citep{6530030}, SINDy \cite{brunton2016discovering} and neural ODEs (NODEs) \citep{chen2018neural}. 

We report the full-system and visible-subsystem errors as 
\begin{align*}
\texttt{REE}_\text{full}&=\frac{\|[\hat{A} \;\; \hat{B}]-[A \;\; B]\|_F}{\frob{[A \;\; B]}}, \\
\texttt{REE}_\text{vis}&=\frac{\|[\hat{A}_V \;\; \hat{B}_V]-[A_V \;\; B_V]\|_F}{\frob{[A_V \;\; B_V]}},
\end{align*}
where $A_V=P_V^\top AP_V, \; B_V=P_V^\top B, \;\hat{A}_V=P_V^\top \hat{A}P_V,$ and $\hat{B}_V=P_V^\top \hat{B}$, using the true visible subspace basis $P_V$ computed from $(A, B, \rx_0)$. Thus, $\texttt{REE}_\text{vis}$ directly measures recovery of the uniquely identifiable subsystem $\left(A\big|_{V(\rx_0)}, B\right)$, while $\texttt{REE}_\text{full}$ reflects the ambiguity induced by unconstrained dynamics on complements of $V(\rx_0)$. Our goal is to validate the theoretical identifiability characterization rather than to propose an estimation method. Thus, using the oracle basis $P_V$ is sufficient for validation purposes. \cref{ssection:non-oracle-V(x0)} reports results under a full data-driven evaluation, where the basis $\hat{P}$ is estimated from the observed trajectory. The results are consistent with the oracle-based evaluation under lower noise levels.

\begin{figure*}[t]
    \centering
    \begin{minipage}{\textwidth}
        \centering
        \includegraphics[width=0.95\linewidth]{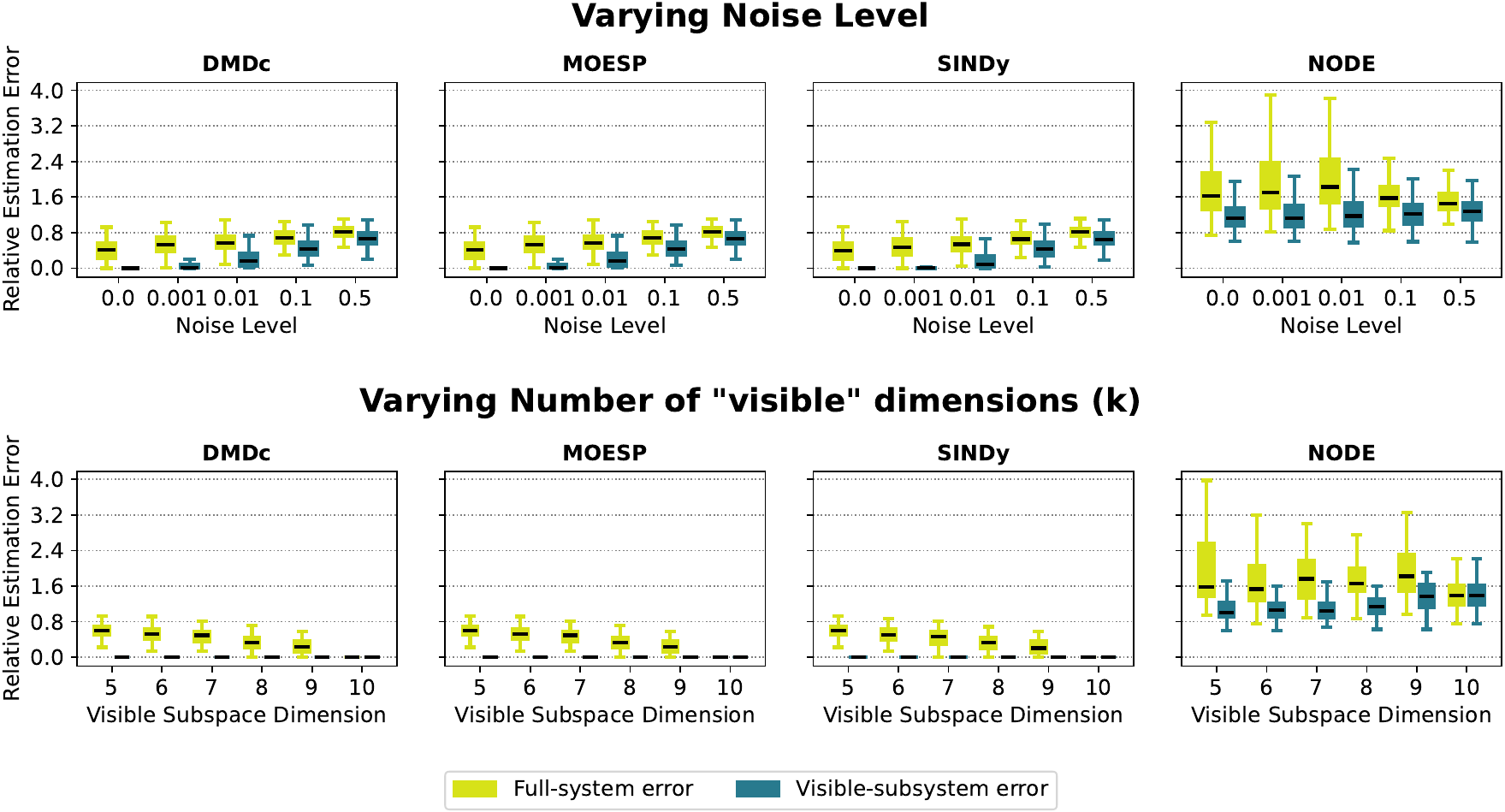}
        \caption{\textbf{Estimation error in full state space vs. visible subspace.} Comparison of the relative estimation error for the full system and the restriction to the visible subspace across four identification methods (DMDc, MOESP, SINDy, and NODE) for a system with dimensionality $n=10$. \textbf{(Top)} Errors for increasing Gaussian observation noise levels $\sigma \in \{0, 10^{-3}, 10^{-2}, 10^{-1}, 0.5\}$. \textbf{(Bottom)} Errors for visibility dimensions $k \in \{5, \dots, 10\}$ in the noise-free regime.}
        \label{fig:fullsys-vs-vissys}
    \end{minipage}
\end{figure*}

\textbf{Results.} \Cref{fig:fullsys-vs-vissys} (top) shows the estimation error of different identification algorithms depending on the observation noise level. In the noise-free setting, DMDc, MOESP, and SINDy recover a model from the experiment-consistent set $[A, B]_\re$ (though not necessarily the true system). In particular, the visible-subspace restriction achieves near-exact recovery, i.e., $\texttt{REE}_\text{vis} \approx 0$, whereas the full-system estimation error $\texttt{REE}_\text{full}$ remains bounded away from zero. As the observation noise increases, both errors increase for DMDc, MOESP, and SINDy, because the noise perturbs the observed trajectory and degrades the conditioning. Nevertheless, across the tested noise regime, the visible-subspace error remains consistently lower than the full-state error.

NODEs behave differently: while the full system error is still larger than the visible-subsystem error, both errors are higher than those of the closed-form estimators across all noise levels.
Moreover, in the noise-free setting, the visible-subspace error does not vanish. This is consistent with the observation that the NODE estimators do not perfectly reconstruct the training trajectory. We hypothesize that this is due to practical training effects: NODEs solve a nonconvex, gradient-based optimization problem coupled to numerical ODE integration, in contrast to the closed-form/least-squares estimators considered here. As the observation noise increases, the gap between $\texttt{REE}_\text{vis}$ and $\texttt{REE}_\text{full}$ for NODEs also becomes smaller. 

\Cref{fig:fullsys-vs-vissys} (bottom) shows the estimation error depending on the dimensionality of the visible subspace. When the visible dimension $k=\operatorname{dim}(V(\rx_0))$ is small, the experiment leaves a large portion of the state space unexcited, and the experiment-consistent class of systems is correspondingly large; this manifests as an increased estimation error $\texttt{REE}_\text{full}$ for all identification methods. As $k$ increases, the experiment constrains a larger portion of the dynamics and $\texttt{REE}_\text{full}$ decreases monotonically for all identification methods. Similar to the results on varying noise levels, DMDc, MOESP and SINDy uniquely recover the identifiable restriction $(A|_{V(\rx_0)}, B)$ (i.e. $\texttt{REE}_\text{vis}$ remains near zero) for all $k$, whereas NODEs do not achieve vanishing $\texttt{REE}_\text{vis}$. 

In \cref{ssection:stress-tests}, we report the full-system and visible-subsystem errors for high-dimensional systems ($n \in \{5, \dots, 100\}$) and for varying sampling rates $\Delta t \in \{0.05, \dots, 2\}$. The results are consistent with the results in \cref{fig:fullsys-vs-vissys}.

\section{Related Work}\label{sec:related-work}

Learning LTI systems from data has received considerable attention \citep{hardt2018gradient,simchowitz2018learning,sarkar2019near}. These works derive finite-time, non-asymptotic bounds on the error of least-squares and gradient-based estimators, characterizing how quickly the system matrices can be recovered as a function of trajectory length and noise level. A common feature is that identifiability is guaranteed by design: the system is driven by i.i.d. process/observation noise (or random inputs) that ensures the state visits all directions, making the estimation well-posed. We address the complementary, logically prior question: when the experiment does not provide such sufficient excitation, which parameters are recoverable at all? 

Classical identifiability for controlled LTIs \citep{kalman1963mathematical, BELLMAN1970329} establishes that controllability plus PE inputs suffice for identifiability. 
Identifiability for autonomous systems has been studied geometrically by \citet{stanhope2014identifiability}, who link global identifiability from an initial state to whether the orbit is confined to a proper $A-$invariant subspace, and by \citet{qiu2022identifiability}, who derive closed-form descriptions of experiment-consistent systems via Jordan decomposition. Our work extends this geometric perspective from autonomous to controlled systems. In this setting, persistent excitation becomes essential for identifiability. Moreover, rather than using a Jordan-decomposition characterization as in \citet{qiu2022identifiability}, we formulate the equivalence class through Krylov subspaces-- an approach which connects directly to controllability and yields simple algebraic tests.

Recent work studies identifiability under hidden confounders \citep{wang2024identifiability} and via multi-environment intervention diversity \citep{rajendran2024interventional}. These address different ambiguities-- unobserved latent variables and representation non-uniqueness under Gaussianity -- rather than the geometric limits imposed by a single experiment. Closest to our setting is the work by  \citet{yu2021controllability}, who show that data-driven predictive control remains equivalent to model-predictive control when controllability fails. Their focus is which \textit{control tasks} remain solvable without controllability; ours is which \textit{parameters} remain identifiable, a logically prior question that characterizes the experiment-induced epistemic uncertainty.

\section{Discussion}\label{sec:discussion}

We characterize experiment-conditional identifiability for controlled LTI systems without assuming controllability-- a setting that classical theory \citep{kalman1963mathematical, willems2005note} and the autonomous-system literature \citep{stanhope2014identifiability, qiu2022identifiability} both leave unresolved. We derive the visible subspace, the state space that an experiment can probe. If the visible subspace is a proper subspace of the state space, the system is only partially identifiable. 
In that case, any system identification method that fits the observed trajectory is blind to the dynamics outside of the visible subspace, and conclusions about those directions are unwarranted regardless of training fit.
Empirically, we show that sparse systems are often uncontrollable, but system identification methods still successfully recover the identifiable subsystem.

\section*{Acknowledgements}
This work has been supported by the German Federal Ministry of Education and Research (Grant: 01IS24082).

\bibliography{bibliography}

\newpage

\onecolumn
\renewcommand\thanks[1]{}
\title{Limits of Learning Linear Dynamics from Experiments\\(Supplementary Material)}
\maketitle

\appendix
\crefalias{section}{appendix}
\crefalias{subsection}{appendix}

\section{Proofs}
\label{ssection:proofs}

\paragraph{Proof: Minimality of $V(\rx_0)$}
\begin{proof}
    \label{proof:minimality-of-V(E)}
    We want to show that $V(\rx_0):=\mathcal{K}_n(A, \Kcore)$ is the strictly smallest $A-$invariant subspace containing both the initial state $\rx_0$ and the input image $\mathrm{Im}(B)$.

    \textbf{1. $V(\rx_0)$ satisfies the conditions.} By the definition of the Krylov sequence, the $0-$th order terms include $\rx_0$ and the columns of $B$, meaning $\{\rx_0\}\cup \operatorname{Im}(B)\subseteq V(\rx_0)$. Furthermore, multiplying any vector in $V(\rx_0)$ by $A$ yields a linear combination of terms up to $A^n\Kcore$. By the Cayley-Hamilton theorem, $A^n$ can be expressed as a linear combination of lower powers $A^0, \dots, A^{n-1}$. Thus, $AV(\rx_0) \subseteq V(\rx_0)$, making it $A-$invariant.

    \textbf{2. Minimality.} Let $S \subseteq \brn{n}$ be any arbitrary $A-$invariant subspace $(AS \subseteq S$) that contains $\{\rx_0\}\cup \mathrm{Im}(B)$. Because 
    $\rx_0 \in S$ and $S$ is $A-$invariant, we necessarily have $A^{k}\rx_0 \in S$ for all $k \in \{0, \dots, n-1\}$. By the same logic, because $\mathrm{Im}(B) \subseteq S$, we have $A^k\mathrm{Im}(B)\subseteq S$  for all $k \in \{0, \dots, n-1\}$. Because $S$ is a linear subspace, it is closed under vector addition and scalar multiplication. Therefore, $S$ must contain all linear combinations of these constituent vectors. Consequently, the entire span of these vectors is contained within $S$:
    \begin{align*}
        \operatorname{span}\left\{\Kcore, \dots, A^{n-1}\Kcore\right\} \subseteq S,
    \end{align*}
    which is exactly $V(\rx_0)\subseteq S$. Since $V(\rx_0)$ is contained within every possible $A-$invariant subspace $S$ that satisfies these conditions, $V(\rx_0)$ is by definition the minimal such subspace.
\end{proof}

\paragraph{Proof (\Cref{lemma:orbit-confinement}): Orbit confinement.}
\begin{proof}
\label{proof:trajectory-confinement}
By \cref{def:visible-subspace}, the visible subspace $V(\rx_0)$ is the smallest $A-$invariant subspace containing $\{\rx_0\} \cup \operatorname{Im}(B)$:
\begin{align*}
    \rx_0 \in V(\rx_0), \quad \operatorname{Im}(B) \subseteq V(\rx_0), \quad AV(\rx_0) \subseteq V(\rx_0).
\end{align*}
Since $V(\rx_0)$ is $A-$invariant, it is also invariant under $A^k$ for any $k \in \mathbb{N}_0$, hence invariant under the matrix exponential $e^{At}=\sum_{k=0}^{\infty}\frac{t^k}{k!}A^k$, i.e., $e^{At}V(\rx_0)\subseteq V(\rx_0)$ for all $t \geq 0$. For any $t \in [0,T]$, the system state is given by $\varphi(t\spbar A, B, \exptuple)=e^{At}\rx_0+\int_0^te^{A(t-s)}B\ru(s) \; ds$. The autonomous response $e^{At}\rx_0$ lies in $V(\rx_0)$ since $\rx_0 \in V(\rx_0)$ and $e^{At}V(\rx_0)\subseteq V(\rx_0)$. For the forced response $\int_0^te^{A(t-s)}B\ru(s) \; ds$, for any $s \in [0,t] \subseteq [0,T]$, the vector $B\ru(s)$ lies in $\operatorname{Im}(B)$, and thus $B\ru(s)\in V(\rx_0)$. Because $V(\rx_0)$ is invariant under $e^{A(t-s)}$, the integrand $e^{A(t-s)}B\ru(s)$ remains entirely within $V(\rx_0)$ for all $s$. Furthermore, since $V(\rx_0)$ is a finite-dimensional closed subspace, the definite integral of this curve also evaluates to a vector inside $V(\rx_0)$. Since both the autonomous and forced parts of the system response lie within $V(\rx_0)$, their sum must also reside in $V(\rx_0)$. Thus, $\varphi(t\spbar A, B, \exptuple)\in V(\rx_0)$ for all $ t \in [0,T]$.
\end{proof}

\paragraph{Proof (\Cref{lemma:block-form-in-VW}): Block upper-triangular form.}
Let $k \coloneqq \operatorname{dim}(V(\rx_0))$ and denote our visible subspace as $V \coloneqq V(\rx_0)$. By the definition of the visible subspace, we know that $\rx_0 \in V$ and $\mathrm{Im}(B) \subseteq V$. Since the columns of $P \in \mathbb{R}^{n \times k}$ form a basis for $V$, there exist unique coordinate representations $\rx_{0,V} \in \mathbb{R}^k$ and $B_V \in \mathbb{R}^{k \times m}$ such that $\rx_0 = P\rx_{0,V}$ and $B = PB_V$. Applying the inverse coordinate transformation $T^{-1}$ (where $T=[P \;\; Q]$), the components along the complementary subspace $W$ vanish entirely. This gives us our transformed initial state and input matrices:
\begin{align*}
    T^{-1}\rx_0 = \begin{bmatrix} \rx_{0,V} \\ 0 \end{bmatrix}, \quad T^{-1}B = \begin{bmatrix} B_V \\ 0 \end{bmatrix}
\end{align*}
Next, we evaluate the state transition matrix in these new coordinates, defined as $\tilde{A} \coloneqq T^{-1}AT$. We can partition $\tilde{A}$ into four blocks: $\tilde{A} = \begin{bmatrix} A_V & A_* \\ A_{21} & A_W \end{bmatrix}$.
Because $V$ is an $A$-invariant subspace ($AV \subseteq V$), applying $A$ to any vector in $V$ yields another vector strictly in $V$. The first $k$ columns of $A T$ correspond to $AP$, which must lie entirely in $\mathrm{Im}(P) = V$. Therefore, these columns have no projection onto the $W$ subspace. This forces the lower-left block to be zero ($A_{21} = 0$), yielding the block upper-triangular form: $T^{-1}AT = \begin{bmatrix} A_V & A_* \\ 0 & A_W \end{bmatrix}$. Finally, we apply \Cref{lemma:orbit-confinement} (Orbit Confinement). Because the trajectory $\rx(t)$ never leaves $V$, its coordinate representation in the $T$ basis takes the form $z(t) \coloneqq T^{-1}\rx(t) = [\xi(t)^\top \;\; 0^\top]^\top$. Substituting this into the transformed ODE $\dot{z}(t) = \tilde{A}z(t) + T^{-1}B\ru(t)$ leaves us with the exact reduced dynamics:

\begin{align*}
    \begin{cases} \dot{\xi}(t) = A_V\xi(t) + B_V\ru(t), \quad \xi(0) = \rx_{0,V} \\ 0 = 0 \end{cases}
\end{align*}

Thus, the observed trajectory is driven exclusively by the sub-system $(A_V, B_V, \rx_{0,V})$ and is completely independent of both the cross-coupling term $A_*$ and the complementary dynamics $A_W$.

\paragraph{Proof: Geometry of the Design Limit.}
\begin{proof}
\label{proof:design-limit}
    We first characterize the design limit $[A, B]_\calE$ by showing the following mutual implication:
\begin{align*}
   &\textbf{to show: } \quad \abtilde \in [A, B]_{\mathcal{E}} \iff \bigbra{ \tilde{A}\big|_{V(\rx_0)}=A\big|_{V(\rx_0)} \land \quad \tilde{B}=B},
\end{align*}
with 
\begin{align*}
    \abtilde \in [A, B]_{\mathcal{E}} \iff \quad \forall \ru \in \calU, t \in [0,T]: \traj{t}{A, B, \rx_0, \ru)}=\traj{t}{\tilde{A}, \tilde{B}, \rx_0, \ru}.
\end{align*}
$\textbf{1. Necessity }\textbf{("}\mathbf{\implies} \textbf{"):}$ Assume that $\abtilde \in [A, B]_{\mathcal{E}} $. By definition of $ [A, B]_{\mathcal{E}}$, the trajectories agree for every admissible input $\ru \in \calU$. In particular, since the zero input is admissible via \cref{assum:input-class-richness}, we may choose $\ru(\cdot)\equiv \mathbf{0}_m$ and obtain equality of the autonomous responses:
\begin{align*}
   \forall t \in [0,T]: \;  e^{At}\rx_0&= e^{\tilde{A}t}\rx_0 .
\end{align*}
Since both sides are real-analytic in $t$, we get
\begin{align}
    \notag
    \forall k \in \mathbb{N}_0: \rightbarwrap{\frac{d^k}{dt^k}e^{At}\rx_0}{t=0}&=\rightbarwrap{\frac{d^k}{dt^k}e^{\tilde{A} t}\rx_0}{t=0} \iff\\
    \notag
    \forall k \in \mathbb{N}_0: \rightbarwrap{A^ke^{At}\rx_0}{t=0}&= \rightbarwrap{\tilde{A}^ke^{\tilde{A} t}\rx_0}{t=0} \iff \\
    \label{eq:homogeneous-state-Markov-parameter-equality}
    &\hspace{-2.4cm}\boxed{\forall k \in \mathbb{N}_0: A^k\rx_0= \tilde{A}^k\rx_0} \;.
\end{align}
Since the equality of autonomous responses is necessary for equivalence, so is equality between forced responses:
\begin{align*}
    \forall t \in [0,T], \; \ru(\cdot) \in \calU: \; \traj{t}{A, B, \rx_0, \ru)}&=\traj{t}{\tilde{A}, \tilde{B}, \rx_0, \ru} \iff  \\
     e^{At}\rx_0 +\int_0^te^{A(t-s)}B\ru(s) \spone ds &= e^{\tilde{A}t}\rx_0 +\int_0^te^{\tilde{A}(t-s)}\tilde{B}\ru(s) \spone ds \; ; \overset{e^{At}\rx_0 = e^{\tilde{A}t}\rx_0}{\iff}  \\
     \int_0^te^{A(t-s)}B\ru(s) \spone ds &= \int_0^te^{\tilde{A}(t-s)}\tilde{B}\ru(s) \spone ds .
\end{align*}
We define the kernels $K(\tau)\coloneqq e^{A\tau}B\;$, $\tilde{K}(\tau)\coloneqq e^{\tilde{A} \tau} \tilde{B}$, with the help of which we rewrite the equality above as:
\begin{align*}
    \forall t \in [0,T], \; \ru(\cdot) \in \calU:& \quad \int_0^t K(t-s)\ru(s) \spone ds= \int_0^t \tilde{K}(t-s)\ru(s) \spone  ds \\
    \forall t \in [0,T], \; \ru(\cdot) \in \calU:\quad& (K*u)(t) = (\tilde{K}*u)(t).
\end{align*}
Via kernel uniqueness \citep{folland1999real}, we get $K\equiv\tilde{K}$. Thus:
\begin{align*}
    \forall t \in [0,T]: \; K(t)&=\tilde{K}(t) \iff \\
     \forall t \in [0,T]: e^{At}B &= e^{\tilde{A} t} \tilde{B} \iff \\
    \notag
    \forall k \in \mathbb{N}_0: \; \rightbarwrap{\frac{d^k}{dt^k} e^{At}B}{t=0} &= \rightbarwrap{\frac{d^k}{dt^k} e^{\tilde{A} t}\tilde{B}}{t=0} \iff \\
    \notag
    \forall k \in \mathbb{N}_0: \; \rightbarwrap{A^{k}e^{At}B}{t=0} &= \rightbarwrap{\tilde{A}^{k}e^{\tilde{A} t}\tilde{B}}{t=0} \iff \\
    &\hspace{-2.4cm}\boxed{\forall k \in \mathbb{N}_0: \; A^kB=\tilde{A}^k\tilde{B}}\;, 
\end{align*}
i.e., the state-Markov parameters agree for all $k \in \mathbb{N}_0$. Especially, we obtain $B = \tilde{B}$ by setting $k=0$.
Remember that the visible subspace is defined as
\begin{align*}
    V(\rx_0):=\operatorname{span}\bigbra{A^k\rx_0, A^kb_j:\; k \in [0, n-1], j \in [m]}.
\end{align*}
Then, the equations $\forall k \in \mathbb{N}_0: A^k\rx_0= \tilde{A}^k\rx_0$ and $\forall k \in \mathbb{N}_0: \; A^kB=\tilde{A}^k\tilde{B}$ imply the equality of visible subspaces
\begin{align*}
    \tilde{V}(\rx_0):=\operatorname{span}\bigbra{\tilde{A}^k\rx_0, \tilde{A}^k\tilde{b}_j:\; k \in [0, n-1], j \in [m]}=V(\rx_0),
\end{align*}
as well as the equality of restrictions: $A\big|_{V(\rx_0)}=\tilde{A}|_{V(\rx_0)}$, and we have already shown that $\tilde{B}=B$ is necessary for experiment-consistency between the systems.

$\textbf{2. Sufficiency }\textbf{("}\mathbf{\impliedby} \textbf{"):}$ Let $\abtilde$ be any block matrix in $\brn{n \times (n+m)}$ such that $\left.\tilde{A}\right|_{V(\rx_0)}=\left. A \right|_{V(\rx_0)}$ and $\tilde{B}=B$. Since $\rightbarwrap{\tilde{A}}{{V(\rx_0)}}=\rightbarwrap{A}{{V(\rx_0)}}$, and ${V(\rx_0)}$ is $A-$invariant \textbf{(*)}, it is also $\tilde{A}-$invariant:
\begin{align*}
    \forall v \in V: \; \tilde{A} v=Av \overset{\textbf{(*)}}{\in } V \iff \tilde{A} V \subseteq V\overset{\text{def.}}{\iff} V \; \text{is} \; \text{$\tilde{A}-$invariant}.
\end{align*}
\textit{All powers agree on $V(\rx_0)$}. Next, we argue that all powers agree on $V(\rx_0)$:
\begin{align*}
    \forall k \in \mathbb{N}_0: \;\rightbarwrap{\tilde{A}^k}{V(\rx_0)}=\rightbarwrap{A^k}{V(\rx_0)}.
\end{align*}
This follows by induction: for $k=0$ it is trivial. If $\tilde{A}^k v=A^k v$, then $A^k v\in V$ by $A$-invariance, and thus
\begin{align*}
    \tilde{A}^{k+1}v=\tilde{A}(\tilde{A}^k v)=\tilde{A}(A^k v)=A(A^k v)=A^{k+1}v,
\end{align*}
where we used $\left.\tilde{A}\right|_V=\left.A\right|_V$ in the third equality. \\
\textit{Exponentials agree on $V(\rx_0)$}. Using the power series for the matrix exponential and $\forall k \in \mathbb{N}_0:\tilde{A}|_{V(\rx_0)}=A|_{V(\rx_0)} \; (*)$, we obtain:
\begin{align*}
    \rightbarwrap{e^{\tilde{A} t}}{V}=\rightbarwrap{\underset{k=0}{\overset{\infty}{\sum}} \frac{\tilde{A}^kt^k}{k!}}{V}=\underset{k=0}{\overset{\infty}{\sum}} \frac{t^k}{k!} \rightbarwrap{\tilde{A}^k}{V}\overset{\text{(*)}}{=}{\overset{\infty}{\sum}} \frac{t^k}{k!} \rightbarwrap{A^k}{V}=\rightbarwrap{\underset{k=0}{\overset{\infty}{\sum}} \frac{A^kt^k}{k!}}{V}=\rightbarwrap{e^{A t}}{V}.
\end{align*}
In particular, since $\rx_0\in V(\rx_0)$ by construction, we get $e^{\tilde{A} t}\rx_0=e^{A t}\rx_0$. \\
\textit{Forced responses coincide.} Using the same argument, we can show that $e^{\tilde{A} (t-s)}\tilde{B} =e^{A (t-s)}B$. Thus, for any piecewise-continuous input $\ru(\cdot) \in \calU$, the forced responses coincide:
\begin{align*}
    \int_0^t e^{\tilde{A} (t-s)}\tilde{B}\ru(s) \spone ds = \int_0^t e^{A (t-s)}B\ru(s) \spone ds.
\end{align*}
Overall, we get for all $t \in [0,T]$ and $\ru(\cdot) \in \calU:$
\begin{align*}
    e^{\tilde{A} t}\rx_0+\int_0^t e^{\tilde{A} (t-s)}\tilde{B}\ru(s) \spone ds &= e^{A t}\rx_0 +\int_0^t e^{A (t-s)}B\ru(s) \spone ds \\
    \iff \varphi(t \spbar \tilde{A}, \tilde{B}, \rx_0, \ru) &= \varphi(t \spbar A, B, \rx_0, \ru).
\end{align*}
This establishes that $[A, B]_\calE$ contains exactly those systems matching $\abpar$ on the visible subspace.
\end{proof}

\paragraph{Proof (\Cref{eq:pbh-test}): Frobenius distance to PBH failure for fixed $\rx_0$.}
\begin{proof}
\label{proof:Frobenius-distance-to-PBH-failure}
    Let $\rx_0\neq 0$ and let $Q\in\mathbb R^{n\times(n-1)}$ have orthonormal columns spanning $\rx_0^\perp$. By the PBH characterization applied to the augmented pair $(A,[\rx_0\ B])$, loss of $\rx_0$-identifiability under perturbations $(\Delta A,\Delta B)$ is equivalent to existence of $\lambda\in\mathbb C$ and $w\neq 0$ such that
    \begin{align*}
        w^\top(\lambda I-(A+\Delta A))=0,\quad w^\top\rx_0=0,\quad w^\top(B+\Delta B)=0.
    \end{align*}
    The constraint $w^\top\rx_0=0$ implies $w=Qz$ for some $z\neq 0$, hence the above is equivalent to
    \begin{align*}
        z^\top\,Q^\top[\lambda I-(A+\Delta A),\ B+\Delta B]=0,
    \end{align*}
    i.e.\ $\rank(Q^\top[\lambda I-(A+\Delta A),\ B+\Delta B])<n-1$ for some $\lambda$. Therefore the distance to PBH failure equals the minimum Frobenius perturbation needed to make $Q^\top[\lambda I-A,\ B]$ lose row rank for some $\lambda$. By Eckart--Young--Mirsky, the minimum Frobenius perturbation that makes a matrix rank-deficient equals its smallest singular value, yielding
    \begin{align*}
        d_{\mathrm{PBH}}(A,B\mid \rx_0)=\inf_{\lambda\in\mathbb C}\ \sigma_{\min}\!\big(Q^\top[\lambda I-A,\ B]\big).
    \end{align*}
    For $\rx_0=0$, one may take $Q=I_n$, recovering the classical distance-to-uncontrollability formula.
\end{proof}

\paragraph{Proof (\Cref{lemma:pe_transfer}): PE transfer from input to state.}
\begin{proof}
    \label{proof:pe-transfer-u2x}
    Fix $V(\rx_0)=\mathcal{K}_n(A, \Kcore)$ with $\dim(V(\rx_0))=k$ and set $A_V \coloneqq A\big|_{V(\rx_0)} \in \brn{k \times k}$, $B_V=B\big|_{V(\rx_0)}\in \brn{k \times m}$. Define the projected joint regressor as $z(t):=\begin{bmatrix}
        \xi(t) \\ \ru(t)
    \end{bmatrix}$, where $\xi(t)=P_V^\top\rx(t) \in \brn{k}$ and $P_V\in \brn{n \times k}$ is an orthonormal basis spanning $V(\rx_0)$. We want to show that, if $\ru : [0,T]\to \brn{m}$ is persistently exciting of order (at least) $k+1$, the Gramian $\mathcal{G}_{[0,T]}=\int_0^Tz(t)z(t)^\top dt$ is positive definite ($\gram \succ 0)$

    Let $w=\begin{bmatrix}\nu\\ \eta\end{bmatrix}\in\mathbb{R}^{k+m}$ (with $\nu \in \brn{k}$ and $\eta \in \brn{m}$). Then:
    \begin{align}
    \label{eq:annihilator}
        w^\top \gram w= \underset{0}{\overset{T}{\int}} w^\top z(t) z(t)^\top w \;\; dt =\underset{0}{\overset{T}{\int}}\|w^\top z(t)\|^2 \;\; dt=\underset{0}{\overset{T}{\int}}\|\nu^\top \xi(t)+\eta^\top \ru(t)\|^2 \;\; dt \geq 0.
    \end{align}
    Thus, $\gram \succ 0$ is equivalent to: 
    \begin{align}
    \label{eq:annihilator-condition}
        \lnot \; \exists \brn{k+m} \ni w\neq \mathbf{0}: \; \forall t \in [0,T], \;\; \nu^\top \xi(t)+\eta^\top \ru(t)=0,
    \end{align}
    i.e., there is no nonzero $w$ for which $\nu^\top \xi(t)+\eta^\top \ru(t)=0$ for all $t \in [0,T]$.

    Assume for contradiction that such a $w=\begin{bmatrix}\nu\\ \eta\end{bmatrix}$ exists. Since $\ru$ is PE of order $k+1$ by assumption, it is $k$ times differentiable, and so is $\xi(t)$.  Differentiating $\nu^\top \xi(t)+\eta^\top \ru(t)=0$ $i$ times for $i \in \{0, \dots, k\}$, we get
    \begin{align}
    \label{eq:differentiated-annihilator}
        \nu^\top \xi^{(i)}(t)+\eta^\top \ru^{(i)}(t)=0 \quad \forall t \in [0,T], \; i \in \{0, \dots, k\}.
    \end{align}
    For $i \geq 1$, by induction from $\dot{\xi}(t)=A_V\xi(t)+B_V\ru(t)$, we have
    \begin{align}
    \label{eq:induction}
        \xi^{(i)}(t)&=A_V^{i}\xi(t)+\underset{j=0}{\overset{i-1}{\sum}} A_V^{i-1-j}B_V\ru^{(j)}(t) 
    \end{align}
    for all $t \in [0,T]$.
    
    Let $p(\lambda)=\lambda^k+c_{k-1}\lambda^{k-1}+\cdots+c_0$ be the characteristic polynomial of $A_V$. By Cayley-Hamilton,
    \begin{align}
    \label{eq:CHT}
    A_V^k+c_{k-1}A_V^{k-1}+\cdots+c_0 I_k=0.
    \end{align}

    Next, we form the linear combination of \Cref{eq:differentiated-annihilator} with the coefficients $c_i$ (with $c_k:=1$):
    \begin{align}
    \label{eq:linearly-combined-differentiated-annihilator}
        \underset{i=0}{\overset{k}{\sum}}c_i\nu^\top \xi^{(i)}(t)+c_i\eta^\top \ru^{(i)}(t)=0 \quad \forall t \in [0,T], \; i \in \{0, \dots, k\}.
    \end{align}
    Substituting \Cref{eq:induction} for $\xi^{(i)}(t)$, we obtain for all $t \in [0,T]$ and all $i \in \{0, \dots, k\}$,
    \begin{align}
        \notag
        \underset{i=0}{\overset{k}{\sum}}c_i\nu^\top \left( A_V^{i}\xi(t)+\underset{j=0}{\overset{i-1}{\sum}} A_V^{i-1-j}B_V\ru^{(j)}(t)  \right)+c_i\eta^\top \ru^{(i)}(t)=0  \longleftrightarrow \\
        \underset{i=0}{\overset{k}{\sum}}c_i \nu^\top  A_V^{i}\xi(t)+ \underset{i=0}{\overset{k}{\sum}}\underset{j=0}{\overset{i-1}{\sum}}c_i A_V^{i-1-j}B_V\ru^{(j)}(t)+ \underset{i=0}{\overset{k}{\sum}}c_i\eta^\top \ru^{(i)}(t)=0  
    \end{align}
    We now rearrange the first term as
    \begin{align}
    \sum_{i=0}^{k} c_i \nu^\top A_V^{i}\xi(t)
    = \nu^\top\Big(\sum_{i=0}^{k} c_i A_V^{i}\Big)\xi(t)
    = \nu^\top p(A_V)\,\xi(t)=0,
    \end{align}
    where the last equality follows from \Cref{eq:CHT}.

    Hence the state $\xi(t)$ is eliminated and we obtain, for all $t\in[0,T]$,
    \begin{align}
    \label{eq:pure-u}
    \sum_{i=0}^{k}\sum_{j=0}^{i-1} c_i \nu^\top A_V^{\,i-1-j}B_V\,\ru^{(j)}(t)
    \;+\;\sum_{i=0}^{k} c_i \eta^\top \ru^{(i)}(t)=0 .
    \end{align}
    With re-indexing, we isolate the $\ru^{(i)}(t)$ terms:
    \begin{align}
    \label{eq:separated-u}
    \notag
    \sum_{j=0}^{k-1}\sum_{i=j+1}^{k} c_i \nu^\top A_V^{\,i-1-j}B_V\,\ru^{(j)}(t)
    \;+\;\sum_{j=0}^{k} c_j \eta^\top \ru^{(j)}(t)=0 \longleftrightarrow \\
    \sum_{j=0}^{k-1}\underbrace{\left(c_j \eta^\top+\sum_{i=j+1}^{k} c_i \nu^\top A_V^{\,i-1-j}B_V\right)}_{:=\gamma_j^\top}\ru^{(j)}(t)+\underbrace{c_k \eta^\top}_{:=\gamma_k^\top} \ru^{(k)}(t) = 0. 
    \end{align}
    Thus, \Cref{eq:differentiated-annihilator} can be compactly expressed as
    \begin{align}
        \label{eq:compact-differentiated-annihilator}
        \underset{j=0}{\overset{k}{\sum}} \gamma_j^\top \ru^{(j)}(t)=0 \quad \forall t \in [0,T]
    \end{align}
    where 
    \begin{align}
        \label{eq:gamma-term}
        \gamma_j^\top =\begin{cases}
            c_j \eta^\top+\sum_{i=j+1}^{k} c_i \nu^\top A_V^{\,i-1-j}B_V & \text{if } \; j \in \{0, \dots, k-1\} \\
            c_k\eta^\top =\eta^\top & \text{if } \; j=k.
        \end{cases},
    \end{align}

    By \cref{def:PEness}, $\ru$ being PE of order $k+1$ means that the only constant coefficient vector $\Gamma^\top=\begin{bmatrix}\gamma_0^\top & \cdots & \gamma_k^\top\end{bmatrix}$
    for which \Cref{eq:pure-u} holds is $\Gamma=0$, enforcing $\gamma_i^\top=\mathbf{0}$ for all $i \in \{0, \dots, k\}$. From $\gamma_k^\top=\eta^\top \overset{!}{=}\mathbf{0}$, we immediately obtain $\eta^\top=\mathbf{0}$.    

    Evaluating \Cref{eq:gamma-term} at $j=k-1$ gives
    \begin{align*}
        \gamma_{k-1}^\top=c_{k-1} \mathbf{0}+\underset{i=(k-1)+1}{\overset{k}{\sum}} c_i\nu^\top A_V^{i-1-(k-1)}B_V=c_k\nu^\top A_V^{0}B_V=\nu^\top B_V \overset{!}{=} \mathbf{0}.
    \end{align*}
    Next, evaluating for $j=k-2$:
    \begin{align*}
        \gamma_{k-2}^\top&=c_{k-2} \mathbf{0}+\underset{i=(k-2)+1}{\overset{k}{\sum}} c_i\nu^\top A_V^{i-1-(k-2)}B_V \\
        & = c_{k-1}\nu^\top A_V^{0}B_V+c_k\nu^\top A_V B_V \\&
        = c_{k-1}\nu^\top B_V+\nu^\top A_V B_V \\
        &=c_{k-1}\mathbf{0}+\nu^\top A_V B_V \\
        &=\nu^\top A_V B_V \overset{!}{=} \mathbf{0}.
    \end{align*}
    Next, evaluating for $j=k-\ell$ for an arbitrary $\ell \in \{1, \dots, k-1\}$, we obtain
    \begin{align*}
        \gamma_{k-\ell}^\top&= \underset{j=1}{\overset{\ell}{\sum}}c_{k-\ell +j}\nu^\top A_V^{j-1}B_V \\
        &=\begin{bmatrix}
            c_{k-\ell +1} & c_{k-\ell +2} & \dots & c_k
        \end{bmatrix} \begin{bmatrix}
            \nu^\top B_V \\ \nu^\top A_VB_V\\ \vdots \\ \nu^\top A_V^{\ell-1}B_V
        \end{bmatrix} \overset{!}{=}\mathbf{0},
    \end{align*}
    from which we confirm that $\nu^\top A_V^{\ell-1}B_V=\mathbf{0}$. Overall, we get
    \begin{align}
    \label{eq:Markov-Annihilation}
        \nu^\top B_V=\nu^\top A_VB_V=\cdots=\nu^\top A_V^{k-1}B_V=0.
    \end{align}
    Since $\eta=0$, the annihilator \Cref{eq:differentiated-annihilator} reduces to $\nu^\top\xi(t)=0$ for all $t$. Evaluating derivatives at $t=0$ gives $\nu^\top\xi^{(i)}(0)=0$ for $i=0,\dots,k-1$.
    Using \Cref{eq:induction} at $t=0$ (and $\xi(0)=\rx_{0,V}$) yields
    \begin{align}
    0=\nu^\top A_V^i \rx_{0,V}+\sum_{j=0}^{i-1}\nu^\top A_V^{i-1-j}B_V\,\ru^{(j)}(0).
    \end{align}
    All terms in the sum vanish by \Cref{eq:Markov-Annihilation}, hence
    \begin{align}
    \label{eq:x0-Annihilation}
    \nu^\top A_V^i \rx_{0,V}=0,\qquad i=0,1,\dots,k-1.
    \end{align}
    
    Combining \Cref{eq:Markov-Annihilation} and \Cref{eq:x0-Annihilation}, we see that $\nu$ annihilates
    \begin{align}
    \mathrm{span}\Big\{A_V^i \rx_{0,V},\ A_V^i\mathrm{Im}(B_V): i=0,\dots,k-1\Big\}.
    \end{align}
    By the \cref{def:visible-subspace} of $V(\rx_0)$ and $\dim(V(\rx_0))=k$, this span equals $\mathbb{R}^k$
    (equivalently, $(A_V,[\rx_{0,V}\ B_V])$ is controllable on $\mathbb{R}^k$), so $\nu=0$.
    Together with $\eta=0$, this contradicts $w\neq 0$.
    Therefore no nonzero $w$ satisfies \Cref{eq:differentiated-annihilator}, and hence $\mathcal{G}_{[0,T]}\succ 0$.
\end{proof}

\paragraph{Proof (\Cref{theorem:bridge}): The Bridge.}
\begin{proof}
    \label{proof:bridge-theorem}
    \textbf{1. Forward Inclusion ($[A, B]_{\mathcal{E}} \subseteq [A, B]_{(\rx_0, \ru)}$):} By definition.

    \textbf{2. Reverse Inclusion ($[A, B]_{(\rx_0, \ru)} \subseteq [A, B]_{\mathcal{E}}$):} Let $\abtilde \in [A, B]_{(\rx_0, \ru)}$ and shorten system responses as $\varphi:=\varphi([0,T]\spbar A, B, \rx_0, \ru)$ and $\tilde{\varphi}:=\tilde{\varphi}([0,T]\spbar \tilde A, \tilde B, \rx_0, \ru)$. To show:
    \begin{align*}
        &\abtilde \in [A, B]_{(\rx_0, \ru)} \implies \abtilde \in [A, B]_{\calE} \\
        & \longleftrightarrow \tilde \varphi = \varphi \implies \tilde A\big|_{V(\rx_0)}=A\big|_{V(\rx_0)} \;\; \land \;\; \tilde B=B.
    \end{align*}
    Since $\tilde \varphi = \varphi$, the time derivatives on $[0,T]$ must match as well:
    \begin{align*}
        \forall t \in [0,T]:\tilde A \rx(t)+\tilde B\ru(t)=A\rx(t)+B\ru(t),
    \end{align*}
    leading to the error equation
    \begin{align*}
        \forall t \in [0,T]: \underbrace{(\tilde A-A)}_{:= \Delta A}\rx(t)+\underbrace{(\tilde B-B)}_{\Delta B}\ru(t)=0.
    \end{align*}
    The implication we're trying to show becomes:
    \begin{align*}
        &\left(\forall t \in [0,T]:
        \Delta A \rx(t)+\Delta B\ru(t)=0 \right) \implies \\
        & \qquad \left( \tilde A\big|_{V(\rx_0)}=A\big|_{V(\rx_0)} \;\; \land \;\; \tilde B=B \right).
    \end{align*}
    Using $\rx(t)=P_V\xi(t)$, we get
    \begin{align*}
        &\left(\forall t \in [0,T]:
        [\Delta AP_V \quad \Delta B][ \xi(t)^\top\ru(t)^\top]^\top=0 \right) \implies \\
        & \qquad \left( \tilde A\big|_{V(\rx_0)}=A\big|_{V(\rx_0)} \;\; \land \;\; \tilde B=B \right).
    \end{align*}
    Let's vector-multiply the left-hand side with $[ \xi(t)^\top\quad \ru(t)^\top]$ and integrate on $[0,T]$:
    \begin{align*}
        \forall t \in [0,T]: \quad
        [\Delta AP_V \quad \Delta B]\underbrace{\underset{0}{\overset{T}{\int}} \begin{bmatrix} \xi(t) \\ \ru(t) \end{bmatrix}
        [ \xi(t)^\top\ru(t)^\top] \;\: dt}_{=\mathcal{G}_{[0,T]}}=0.
    \end{align*}
    The implication we're trying to show becomes:
    \begin{align*}
        &\big(\forall t \in [0,T]:[\Delta AP_V \quad \Delta B]\mathcal{G}_{[0,T]}=0 \big) \implies \\
        &\qquad \left( \tilde A\big|_{V(\rx_0)}=A\big|_{V(\rx_0)} \;\; \land \;\; \tilde B=B \right).
    \end{align*}
    Since the trajectory is \emph{informative} (Def. \ref{def:informative_trajectory}), the Gramian $\mathcal{G}_{[0,T]}$ is positive definite and invertible. Multiplying by $\mathcal{G}_{[0,T]}^{-1}$ from the right yields:
    \begin{align*}
        [\Delta AP_V \quad \Delta B] = 0 \implies \Delta AP_V = 0 \;\; \land \;\; \Delta B = 0.
    \end{align*}
    The condition $\Delta A P_V = 0$ is equivalent to $(\tilde{A}-A)|_{V(\rx_0)} = 0$, or $\tilde{A}|_{V(\rx_0)} = A|_{V(\rx_0)}$. Combined with $\tilde{B}=B$, it follows directly from the geometric characterization of the design limit (\cref{proof:design-limit}) that $\abtilde \in [A, B]_\calE$.
\end{proof}

\paragraph{Proof (\Cref{theorem:identifiability}): Experiment-Conditional Identifiability.}
\begin{proof}
\label{proof:identifiability-1}
\textbf{(a)} Because the realized trajectory is informative, \cref{theorem:bridge} guarantees that $[A, B]_\re=[A, B]_\calE$. By the geometric characterization of the design limit proven in \cref{proof:design-limit}, $[A, B]_\calE$ consists exactly of those systems where $\tilde{A}\big|_{V(\rx_0)}=A\big|_{V(\rx_0)}$ and $\tilde{B}=B$. Thus, the \emph{restricted} dynamics are uniquely determined by the single experiment $\re$. \textbf{(b)} Follows directly from \textbf{(a)}. \textbf{(c)} Full identifiability requires $[A, B]_\re=\{(A, B)\}$. From the parametrization in \textbf{(b)}, the experiment-consistent set must have zero degrees of freedom, i.e., $n(n-k)=0$. This is equivalent to $k=\operatorname{dim}(V(\rx_0))=n$, meaning $V(\rx_0)=\brn{n}$. 
\end{proof}

\paragraph{Proof: Invariants under ZOH Sampling}
\begin{proof}
\label{proof:ZOH-invariance}

\textbf{1. Exact ZOH dynamics.}
By variation of constants, for $u(t)\equiv u[j]$ on $[j\Delta t,(j{+}1)\Delta t)$,
\begin{align*}
    \rx((j+1)\Delta t)=e^{A\Delta t}\rx(j\Delta t)+\Big(\int_0^{\Delta t}e^{As}\,ds\Big)B\ru[j],
\end{align*}
so $\rx[j{+}1]=A_d \rx[j]+B_d \ru[j]$ with $A_d=e^{A\Delta t}$ and $B_d=\int_0^{\Delta t}e^{As}B\,ds$
(e.g., \citep{chen2012optimal}).
\textit{Two basic ingredients:}
Cayley–Hamilton implies $A_d^j=e^{jA\Delta t}$ is a polynomial in $A$; hence $\mathcal{K}_n(A_d,w)\subseteq \mathcal{K}_n(A,w)\quad\text{for any }w$.
Also $B_d=g(A)B$ with $g(z)=(e^{z\Delta t}-1)/z$ (and $g(0)=\Delta t$).

\textbf{2. Reachable space is preserved.}
From $\operatorname{span}\{A_d^kB_d\}\subseteq \operatorname{span}\{A^kB_d\}$ with $w=B_d$,
$\operatorname{span}\{A_d^kB_d\}\subseteq \operatorname{span}\{A^kB_d\}$.
If $\Delta t$ is nonpathological on the \emph{reachable spectrum} $\sigma_r(A,B)$,
then $g(\lambda)\neq 0$ for all $\lambda\in\sigma_r(A,B)$, so $g(A)$ is invertible on
$\mathcal{R}_0^{ct}:=\operatorname{span}\{A^kB\}$ and
\begin{align*}
    \operatorname{span}\{A^kB_d\} =\operatorname{span}\{A^k g(A)B\} =\operatorname{span}\{A^kB\} =\mathcal{R}_0^{ ct},
\end{align*}
giving $\mathcal{R}_0^{dt}\subseteq \mathcal{R}_0^{ ct}$.
For the reverse inclusion, restrict attention to $\mathcal{R}_0^{ ct}$.
On this subspace, the eigenvalues of $A$ are precisely $\sigma_r(A,B)$, which
remain distinct under $z\mapsto e^{\Delta t z}$ by the non-pathological sampling assumption.
Since $f(z)=e^{\Delta t z}$ is analytic with $f'(\lambda)=\Delta t e^{\lambda\Delta t}\neq 0$,
the Jordan block sizes for $A$ and $A_d=f(A)$ on $\mathcal{R}_0^{ ct}$ coincide
(\cite[§3]{higham2008functions}). Consequently, for every polynomial $p$ there exists a polynomial $q$
such that $p(A)\big|_{\mathcal{R}_0^{ ct}} \;=\; q(A_d)\big|_{\mathcal{R}_0^{ ct}}$.
Applying both sides to $B_d\in\mathcal{R}_0^{ ct}$ yields
$\operatorname{span}\{A^kB_d\}\subseteq \operatorname{span}\{A_d^kB_d\}$, hence
$\mathcal{R}_0^{ ct}\subseteq \mathcal{R}_0^{ dt}$.
Combining the two inclusions gives $\mathcal{R}_0^{ dt}=\mathcal{R}_0^{ ct}$
(see also \citep{362892,kreisselmeier2002sampling,braslavsky1995frequency} for closely related criteria).

\textbf{3. Visible subspace is preserved.}
From $\operatorname{span}\{A_d^kB_d\}\subseteq \operatorname{span}\{A^kB_d\}$ with $w\in\{x_0\}\cup\operatorname{col}(B)$ and $B_d=g(A)B$ we get the easy direction
$V_{ dt}(x_0)\subseteq V_{ ct}(x_0)$.

For the reverse inclusion on the \emph{autonomous} part $\mathcal{K}_n(A,x_0)$,
assume $\lambda\mapsto e^{\lambda\Delta t}$ is injective on all of $\sigma(A)$.
By the same analytic/Jordan-block argument as above (now on the full space),
for each polynomial $p$ there is a $q$ with $p(A)x_0=q(A_d)x_0$, hence
$\mathcal{K}_n(A,x_0)\subseteq \mathcal{K}_n(A_d,x_0)$.
Together with (ii), which gives $\mathcal{K}_n(A,B)=\mathcal{K}_n(A_d,B_d)$,
we conclude $V_{ ct}(x_0)\subseteq V_{ dt}(x_0)$, and thus equality.

\textbf{4. Informativeness under ZOH.} To bridge the continuous-time informativeness to ZOH inputs, we rely on the exact discrete-time dynamics on the visible subspace $V(\rx_0)$:
\begin{align*}
    \xi[j+1]=A_{d, V}\xi [j]+B_{d, V}\ru[j],
\end{align*}
where $A_{d,V}=e^{A_V\Delta t}$ and $B_{d,V}= \int_0^{\Delta t}e^{A_Vs}B_V \; ds$. By Step 3, if $\Delta t$ is non-pathological, $(A_{d,V}, [\rx_0 \;\; B_{d,V}])$ is a controllable pair on $\brn{n}$ (where $k=\operatorname{dim}(V(\rx_0))$). 

For discrete-time systems, an input sequence $\ru[0], \dots, \ru[N-1]$ is persistently exciting of order $r$ if its Hankel matrix $\mathcal{H}_r(\ru)$ has full row rank:
\begin{align*}
    \operatorname{rank} 
    \begin{bmatrix}
    \ru[0] & \ru[1] & \dots & \ru[N-r] \\ 
    \ru[1] & \ru[2] & \dots & \ru[N-r+1] \\ 
    \vdots & \vdots & \ddots & \vdots \\ 
    \ru[r-1] & \ru[r] & \dots & \ru[N-1] 
    \end{bmatrix} = rm.
\end{align*}
Let$z[j]:=[\xi[j]^\top , \ru[j]^\top]^\top$. We want to show that if $\ru$ is PE of order $k+1$, the discrete Gramian $\mathcal{G}_d=\sum_j^{N-1}z[j]z[j]^\top$ is positive definite.

Assume for contradiction that $\mathcal{G}_d$ is singular. Then, there exists a nonzero $w=[\nu^\top, \eta^\top]^\top$ such that $w^\top z[j]=0$ for all $j$, meaning $\nu^\top \xi[j]+\eta^\top \ru[j]=0$. Using the discrete-time transition equation $\xi[j+1]=A_{d, V}\xi [j]+B_{d, V}\ru[j]$, we can express future states purely in terms of the initial state and inputs. Substituting $\xi[j]$ into the annihilator and shifting forward in time generates a system of linear equations identical in structure to the continuous-time derivative annihilator in \cref{eq:differentiated-annihilator}. By applying the Cayley-Hamilton theorem to $A_{d,V}$, we can linearly combine $k+1$ consecutive time shifts to eliminate the state $\xi[j]$ entirely. This leaves a linear combination of strictly the inputs $\ru[j], \dots, \ru[j+k]$ equating to zero for all $j$. However, because the input is PE of order $k+1$, the rows of the Hankel matrix are linearly independent, meaning no such non-trivial linear combination can exist. This forces $\nu=0$ and $\eta=0$, contradicting $w\neq 0$. Thus, $\mathcal{G}_d\succ 0$, and the trajectory is informative.

\end{proof}

\section{Additional Experiments}
\label{ssection:additional-experiments}

\subsection{Stress Tests}
\label{ssection:stress-tests}
The experiments in \S\ref{ssection:subsystem-recovery} fix $n=10$, $T=80$, and stratify by visibility $k \in \{5, \dots, 10\}$. We verify that the quantitative finding $\text{REE}_\text{vis}\ll \text{REE}_\text{full}$ is not an artifact of these specific settings. We run two stress tests with DMDc, varying respectively the state dimension at fixed $k=5$ and the sampling step $\Delta t$.

\paragraph{Varying the sampling step (\Cref{fig:stress-tests-deltat}).} We sample 
$50$ baseline continuous-time pairs $(A_c, B_c)$ from the same sparse-truncated Gaussian ensemble used in \S\ref{ssection:subsystem-recovery}, and discretize at $\Delta t \in \{0.05, 0.1, 0.25, 0.5, 1 , 2\}$. We restrict the ensemble to partial visibility cases ($k <n, n=10, m=2$). Noise-free, length$-T=80$ trajectories are simulated under a persistently exciting input and DMDc is applied. The curves show the mean $\pm \text{ std}$ of $\text{REE}_\text{full}$ (red) and $\text{REE}_\text{vis}$ (blue). The visible subsystem error remains near zero, while the full-system error stays bounded away from zero, demonstrating that the structural identifiability gap predicted by Theorem 3.7 is invariant to the sampling rate.

\begin{figure}[H]
    \centering
    \includegraphics[width=0.475\linewidth]{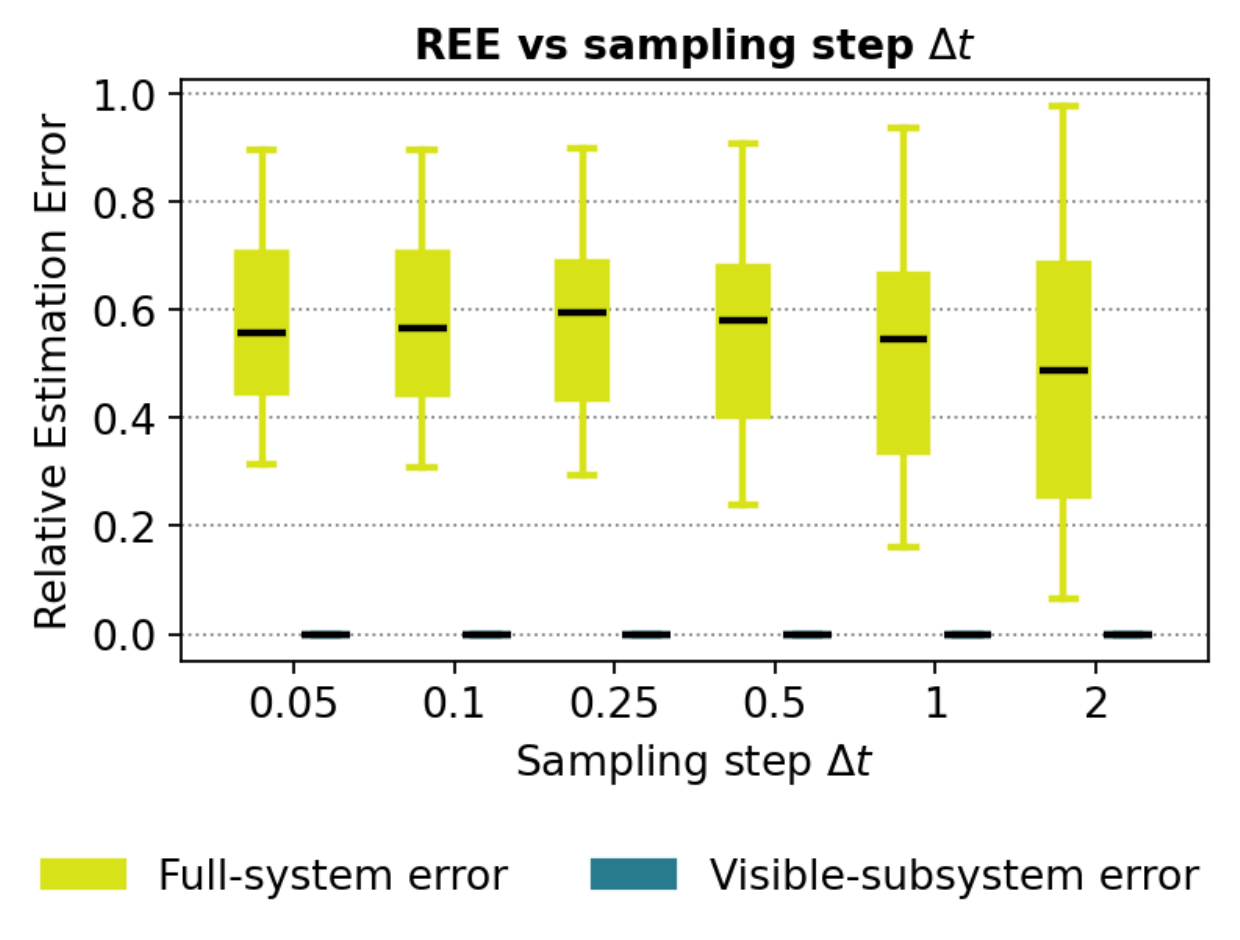}
    \caption{\textbf{Relative estimation error as a function of the sampling step $\Delta t$}. }
    \label{fig:stress-tests-deltat}
\end{figure}

\paragraph{Varying the state dimension (\Cref{fig:stress-tests--ndim}).} For each $n \in \{5, 10, 15, \dots, 100\}$, we draw $24$ systems, simulate noise-free, length$-T=80$ trajectories under a PE input, and run DMDc. Curves show the mean (markers) $\pm \text{ std}$ (whiskers) for $\text{REE}_\text{full}$ (red) and $\text{REE}_\text{vis}$ (blue). The green curve plots the fraction of invisible dimensions. The visible subspace error stays near zero while the full-system error tracks this fraction, as predicted by \cref{theorem:identifiability}.

\begin{figure}[H]
    \centering
    \includegraphics[width=0.75\linewidth]{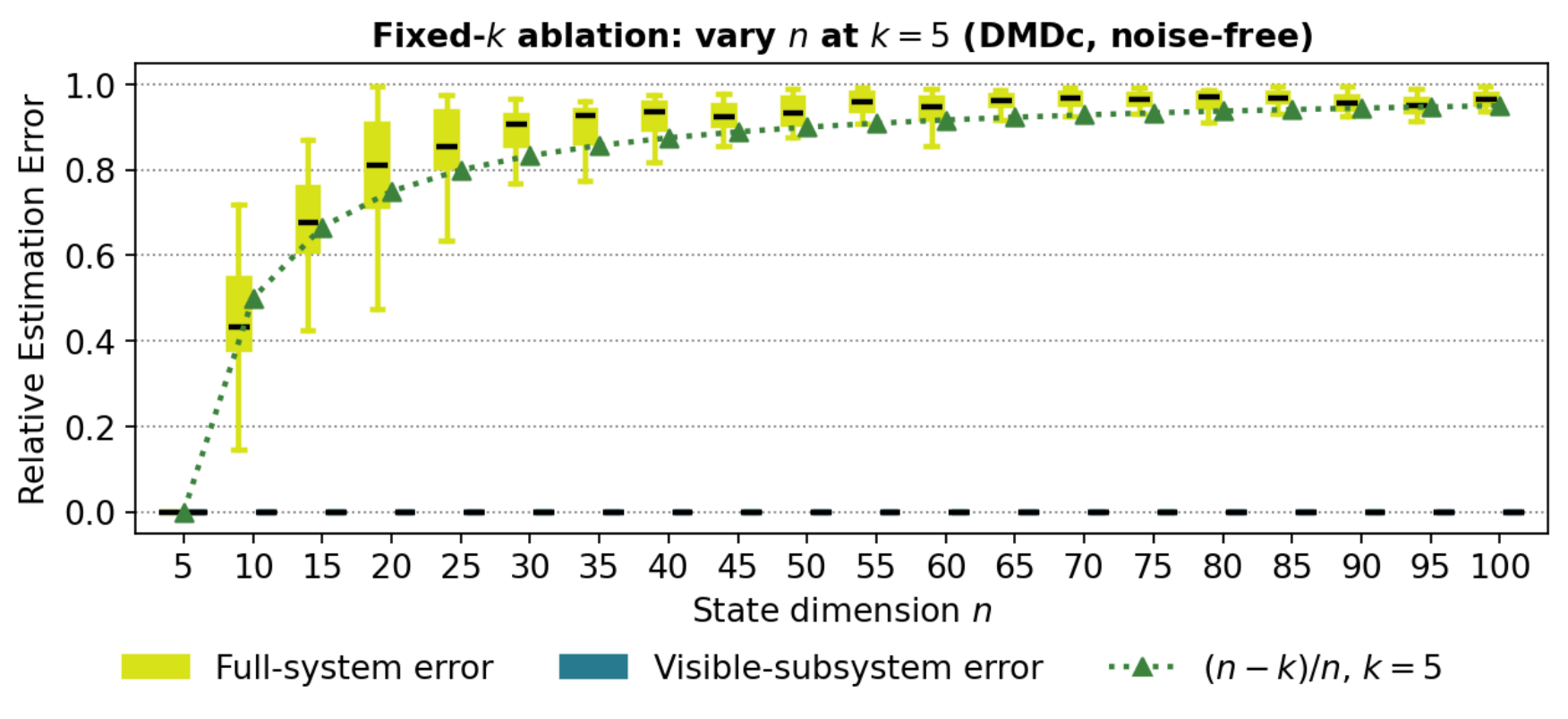}
    \caption{\textbf
    {Relative estimation error as a function of the state dimension $n$ at fixed visible subspace dimension $k=5$.}}
    \label{fig:stress-tests--ndim}
\end{figure}

\paragraph{Results}. Across both stress tests, $\text{REE}_\text{vis}$ remains at machine precision while $\text{REE}_\text{full}$ stays bounded away from zero. In the $\Delta t$ sweep (\Cref{fig:stress-tests-deltat}), the error gap is invariant to the sampling rate, consistent with the visible subspace preservation under non-pathological ZOH sampling established in \Cref{ssection:proofs}. In the fixed$-k$ ablation (\Cref{fig:stress-tests--ndim}), $\text{REE}_\text{full}$ tracks the reference $(n-k)/n$ closely, confirming that the full-system error is dominated by the structurally unidentifiable component. 

\subsection{Data-Driven Visibility Estimation}
\label{ssection:non-oracle-V(x0)}
The experiments in \S\ref{ssection:subsystem-recovery} compute $\vissub$ via the Krylov sequence $\calK_n(A, \Kcore)$, which presumes access to the true system matrices $\abpar$. We now examine whether $\vissub$ can be recovered from the observed trajectory alone.

\begin{figure}[H]
    \centering
    \includegraphics[width=0.95\linewidth]{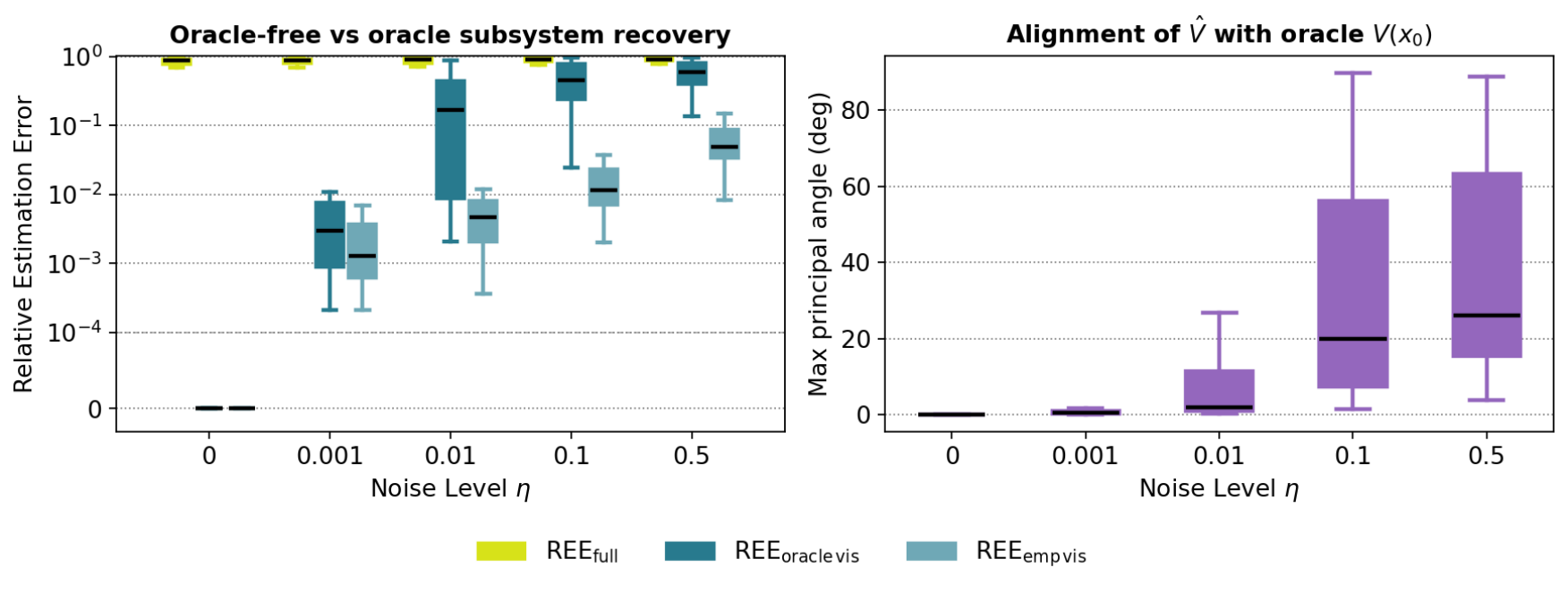}
    \caption{\textbf{Data-driven recovery of the visible subspace and subsystem.} Systems are drawn from the same construction as in other results, with $n=20, k=5, m=2$; the dynamics are simulated for $T=80$ steps under additive Gaussian observation noise, with $\eta \in \{0, 10^{-3}, 10^{-2}, 10^{-1}, 0.5\}$. \textbf{Left.} Three relative estimation errors: $\text{REE}_\text{full}$, $\text{REE}_\text{oracle-vis}$ (projection onto oracle-$V(x_0)$, $\text{REE}_\text{emp-vis}$ (projection onto $\hat{V}$ estimated from data). \textbf{Right.} Maximum principal angle with $\hat{V}$ and $V(x_0)$.}
    \label{fig:oracle-free}
\end{figure}

\paragraph{Empirical visible basis.}
Let $X=[\rx[0], \dots, \rx[T-1]] \in \brn{n \times T}$, with truncated SVD $X=U\Sigma V^\top$ and singular values $\sigma_1 \geq \sigma_2\geq \dots \geq \sigma_n \geq 0$. We define the empirical visible dimension $\hat{k}$ via
\begin{align*}
    \hat{k}=\begin{cases}
        \underset{j \in [n-1]}{\operatorname{argmax}}\frac{\sigma_j}{\sigma_{j+1}} & \text{s.t. } \;\sigma_j > \tau \sigma_\text{max} ,\\
        \operatorname{max} \{j: \sigma_j>\tau \sigma_\text{max}\} & \text{otherwise},
    \end{cases}
\end{align*}
where $\tau=10^{-10}$. The empirical visible basis $\hat{P}\in \brn{n \times \hat{k}}$ consists of the corresponding leading columns of $U$. We then recompute the visible subsystem error using $\hat{P}$ instead of $P_V$:
\begin{align*}
    \texttt{REE}_\text{emp-vis}&=\frac{\|[\hat{P}^\top\hat{A}\hat{P}, \;\; \hat{P}^\top\hat{B}]-[\hat{P}^\top A\hat{P}, \;\; \hat{P}^\top B]\|_F}{\frob{[\hat{P}^\top A \hat{P}, \;\; \hat{P}^\top B]}}.
\end{align*}
As an alignment diagnostic, we also report the maximum principal angle $\theta_\text{max}(\hat{P}, P_V)$ in degrees, computed as $\operatorname{arccos}(\sigma_\text{min}(P_V^\top \hat{P}))$.

\paragraph{Results.} (Left) Without noise ($\eta=0)$, the empirical basis matches the oracle and $\text{REE}_\text{oracle-vis}$ and $\text{REE}_\text{emp-vis}$ both vanish. As the noise level increases, both $\text{REE}_\text{oracle-vis}$ and $\text{REE}_\text{emp-vis}$ degrade; $\text{REE}_\text{emp-vis}$ is affected more severely. (Right) At $\eta=0$, the empirical- and oracle- visible subspaces align perfectly. Increasing noise levels gradually degrade the alignment.

\section{Reproducibility}
\label{ssection:experimental-setup}

\paragraph{Identifiability in Sparse Ensembles (§4.1).} We fix state dimensions $n \in \{2, \dots, 10\}$ and input dimension $m=2$. For each density value $p \in \{0, 0.1, \dots, 1.0\}$, we draw i.i.d. random systems $(A, B) \in \brn{n \times n} \times \brn{n \times m}$ from a sparse random ensemble as described in \cref{ssection:sparse-systems}. We track the realized densities:
\[
\delta_\tau(X) :=\frac{1}{nm}\underset{i \in [n], j \in [m]}{\sum} \mathbf{1\mathbf{\{}} |X_{ij}| > \tau \mathbf{\}},
\]
where $X \in \brn{n \times m}$ and $\tau =10^{-12}$. We label a system as controllable iff $\operatorname{rank}(\mathcal{C}_n(A, B))=n$, where $\mathcal{C}_n(A, B)=[B, AB, \dots, A^{n-1}B]$. For each grid cell $(n, p)$, we draw $S=1000$ independent systems $(A_s, B_s)$ and estimate the uncontrollable fraction $\frac{1}{S}\sum_{s=1}^S \mathbf{1 \{} \operatorname{rank}(\mathcal{C}_n(A, B))<n \mathbf{\}}$.

For each target initial state density $p_{\rx_0} \in \{0.25, 0.5, 0.75, 1.0\}$, we draw $N_{\rx_0}=100$ independent initial conditions as follows: (i)~ sample $v \sim \mathcal{N}(0, I_n)$ and normalize $\rx_0 \gets v/\|v\|_2$; (ii)~ if $p_{\rx_0}<1$, sample i.i.d. Bernoulli mask entries $m_i \sim \operatorname{Ber}(p_{\rx_0})$, set $\rx_0 \gets \rx_0 \odot m$, and renormalize to unit norm (provided $\|\rx_0\|_2 >0$). Let $d(A, B\spbar \rx_0)$ denote our PBH score as defined in \Cref{eq:pbh-test}. For each $\abpar$ and each $\rx_0$, we compute $d(A, B \spbar \rx_0)$, and define a binary identifiability indicator
\[
\operatorname{Id}_\varepsilon(A, B\spbar \rx_0):= \mathbf{1\{} d(A, B \spbar \rx_0)>\varepsilon \mathbf{\}}, \quad \varepsilon=10^{-6}.
\]
We report the aggregated identifiability scores for each $p_{\rx_0}$. All results are produced under a fixed RNG seed $\operatorname{rng}=12345$.

\paragraph{Simulation model (§4.2).} We simulate \Cref{eq:LTI} using Euler integration with fixed step size $dt=1$, over a horizon $T=80$. We generate control inputs as i.i.d. Gaussian sequences for each input channel $i \in [m]$, then normalize by each channel's empirical standard deviation. This yields a persistently exciting input of length $T$. In the NODE estimator, the discrete control sequence is converted to continuous time via zero-order-hold (ZOH).

\paragraph{Baseline Estimators (§4.2).} Beyond the high-level experimental design in §4, we fix the following implementation-level choices for the baseline estimators. DMDc was implemented as the standard least-squares estimator of $X_{1}=[\rx[1], \dots, \rx[T]]$ onto $[X_0;\; U_0]$ (with $X_0=[\rx[0], \dots, \rx[T-1]]$ and $U_0=[\ru[0],\dots, \ru[T-1]]$) to estimate $\abpar$. In our full-state, noise-free setting, MOESP reduces to the same least-squares solution apart from minor differences in numerical behavior. For SINDy, we implemented a linear variant with sequentially thresholded least squares (STLSQ) without relying on PySINDy: coefficients were iteratively fitted and hard-thresholded with sparsity parameter $\lambda=0.05$ for $8$ STLSQ iterations. 

\textbf{NODEs (§4.2).} Neural ODEs \citep{chen2018neural} use a parameterized function $f_{\theta}$ to approximate the underlying ODE of the system in \Cref{eq:LTI}. While $f_{\theta}$ is typically implemented as a neural network, in our sparse linear setting we parameterize it directly by a matrix pair $(A_{\mathrm{eff}},B_{\mathrm{eff}})$ and write $\dot{\rx}(t)=f_{\theta}(\rx(t),\ru(t))=A_{\mathrm{eff}}\rx(t)+B_{\mathrm{eff}}\ru(t)$.  We obtain a continuous control $\ru(t)$ from the control sequence via a zero-order hold. Given the observed initial state $\rx(0)=\rx_0$ and observation times $\{t_k\}_{k=0}^{T-1}$, we predict the trajectory $\hat{X}=\{\hat{\rx}(t_k)\}$ by numerically integrating the ODE using Euler integration. The parameters are learned by minimizing the mean squared reconstruction error $\mathcal{L}_{\mathrm{mse}}=\frac{1}{T}\sum_{k=0}^{T-1}\bigl\|\hat{\rx}(t_k)-\rx(t_k)\bigr\|_2^2$  To encourage sparsity, we parameterize the effective matrices using sigmoid gates with annealed temperature $\tau$, $A_{\mathrm{eff}} = A \odot \sigma(M_A/\tau)$, $B_{\mathrm{eff}} = B \odot \sigma(M_B/\tau)$ and add an $\ell_1$-style penalty on the average gate activations, $ \mathcal{L}=\mathcal{L}_{\mathrm{mse}}+\lambda_A\,\mathbb{E}\!\left[\sigma(M_A/\tau)\right] +\lambda_B\,\mathbb{E}\!\left[\sigma(M_B/\tau)\right]$, where $\odot$ denotes elementwise multiplication. For each system, we estimate a neural ODE with $(\lambda_A, \lambda_B) \in \{(0.1, 0.1), (0.01, 0.01), (0,0)\}$ and choose the value that leads to the smallest reconstruction error of the observed trajectory. All models are trained using RMSprop with a $0.001$ learning rate and $15000$ gradient update steps.

\paragraph{Subsystem Recovery (§4.2).} We generate LTI systems with input dimension $m=2$, sampling sparse-continuous matrices with density $p=0.1$ and $n=10$. We stabilize each $A$ by rescaling to spectral radius $\rho(A) \leq 0.95$ and curate uncontrollable systems. For each curated $\abpar$, we sample initial conditions as follows: (i)~ Sample $v \sim \mathcal{N}(0, I_n)$, (ii)~ If $\|v\|_2 < 10^{-14}$, resample $v$, (iii)~ Set $\rx_0 \gets v/\|v\|_2$. We compute the visible subspace $V(\rx_0)$ via the Krylov matrix $\mathcal{K}_n(A, \Kcore)$ and SVD thresholding with relative tolerance $10^{-10}$, and stratify by the visible dimension $k \in \{5, \dots, 10\}$. For the results presented in \cref{fig:fullsys-vs-vissys} (top), we add i.i.d. Gaussian observation noise to the simulated trajectories with standard deviation $\sigma \in \{0, 10^{-3}, 10^{-2}, 10^{-1}, 0.5\}$. We report $\texttt{REE}_\text{full}$ and $\texttt{REE}_\text{vis}$ using the true visible basis $P_V$ computed from $(A, B, \rx_0)$ with $A_V=P_V^\top A P_V$ and $B_V=P_V^\top B$, and aggregate results over the stratified triple set ($270$ in total, $45$ per $k \in \{5, \dots, 10\}$).

The code for all experiments is available at \href{https://github.com/aiulus/limid}{\texttt{github.com/aiulus/limid}}.

\section{Example}
\label{ssection:example}

\textbf{(experiment-consistency).}
    We illustrate the theory with the following example.
    \begin{align*}
        A=\begin{bmatrix}
            1 & 1 & 0 \\
            0 & 2 & 1 \\
            0 & 0 &3
        \end{bmatrix}, \; B=\begin{bmatrix}
            1 & 0\\
            2 & 1\\
            0 & 0
    \end{bmatrix}
    \end{align*}
    and the initial states
    \begin{align*}
        \rx_0^{(a)}=\begin{bmatrix}0\\0\\1\end{bmatrix}\quad(\text{``good''}),\qquad
        \rx_0^{(b)}=\begin{bmatrix}1\\-1\\0\end{bmatrix}\quad(\text{``bad''}).
    \end{align*}
    A convenient left–eigenbasis of $A$ is
    \begin{align*}
        w_1=\begin{bmatrix}2\\-2\\1\end{bmatrix},\quad
    w_2=\begin{bmatrix}0\\-1\\1\end{bmatrix},\quad
    w_3=\begin{bmatrix}0\\0\\1\end{bmatrix},
    \end{align*}
    for which
    \begin{align*}
        w_1^\top B=[-2,\,-2],\ \ w_2^\top B=[-2,\,-1],\ \ w_3^\top B=[0,\,0],\qquad
    w_3^\top\rx_0^{(a)}=1\neq0,\ \ w_3^\top\rx_0^{(b)}=0.
    \end{align*}
    Thus the third left–eigenvector is completely unexcited by $B$, exactly mirroring the $m{=}1$ base case.
    \paragraph{Visible/hidden decomposition and the secondary system.}
    Let $V\coloneqq\calK(A,\,[\rx_0^{(b)}\ \ B])$ and choose a complementary $W$ so that $\brn{3}=V\oplus W$.
    Because $\rx_0^{(b)}$ has zero third component and the third row of $B$ vanishes, one checks $V=\{(x_1,x_2,0):x_1,x_2\in\brn{}\}$ and $W=\mathrm{span}\{e_3\}$. Define
    \begin{align*}
        \tilde A\ \coloneqq\
        \begin{bmatrix}
        1 & 1 & 0\\
        0 & 2 & 0\\
        0 & 0 & 4
        \end{bmatrix},\qquad
        \tilde B\ \coloneqq\ B.
    \end{align*}
    Then $\tilde A|_{V}=A|_{V}$ and $V$ is $\tilde A$–invariant. Consequently, for any input and any trajectory staying in $V$ (in particular, those from $\rx_0^{(b)}$), the responses of $(A,B)$ and $(\tilde A,\tilde B)$ coincide; while for initial conditions with a $W$–component (e.g. $\rx_0^{(a)}$) the responses differ.
    \paragraph{Identifiability scores.}
    Let $\calK_3\!\coloneqq\!K\!\big(A,[\rx_0\ B]\big)$ and let $\mu_i(A,B\spbar\rx_0)\coloneqq \|w_i^\top[\rx_0\; B]\|_2\big/\bigl(\|w_i\|_2\;\|[\rx_0\; B]\|_2\bigr)$ with $\mu_{\min}=\min_i\mu_i$.
    For our two initial states:
    \begin{align*}
        \begin{array}{c|cccc}
        \text{initial state} & \operatorname{rank}(\calK_3)  & \mu_{\min} & \text{PBH margin} & \lambda_{\min}\big(W_T([\rx_0\ B])\big)\\
        \hline
        \text{good }\rx_0^{(a)} & 3  & \approx0.41 & \approx0.77 &  \approx8307.33 \\
        \text{bad }\rx_0^{(b)} & 2  & 0 & =0 & =0
        \end{array}
    \end{align*}
    Here "PBH margin" denotes $d_{\mathrm{PBH}}(A,B\mid\rx_0)$ as defined in \cref{eq:pbh-test}, and $W_T$ is the finite-horizon controllability Gramian of the augmented pair $(A, \Kcore)$.
\paragraph{How $(\tilde A,\tilde B)$ is picked.}
The construction enforces $\tilde B=B$ and $\tilde A|_V=A|_V$ on the \emph{visible} subspace
\begin{align*}
    V=\calK\!\big(A,\,[\rx_0^{(b)}\ \ B]\big)=\{(x_1,x_2,0)\},
\end{align*}
while freely modifying the \emph{hidden} block on $W=\mathrm{span}\{e_3\}$ (here: moving the eigenvalue $3\to4$ and zeroing the $V\!\to\!W$/$W\!\to\!V$ couplings). This preserves all single-trajectory data from $\rx_0^{(b)}$ yet changes the responses from $\rx_0^{(a)}$, illustrating the equivalence-class mechanism.

\section{Notation}
\label{app:notation}
We summarize the key mathematical notation used throughout the paper below.

\begin{longtable}{ll}
\toprule
\textbf{Symbol} & \textbf{Description} \\
\midrule
$\rx(t) \in \mathbb{R}^n$ & State vector \\
$\ru(t) \in \mathbb{R}^m$ & Control input \\
$(A,B) \in \Omega$ & Unknown system matrices \\
$\Omega \subseteq \mathbb{R}^{n \times n} \times \mathbb{R}^{n \times m}$ & Parameter space \\
$\mathcal{E} = (\rx_0, \mathcal{U})$ & Experiment class \\
$\re = (\rx_0, \ru)$ & Single experiment (realization) \\
$\varphi(t \,|\, A, B, \rx_0, \ru)$ & State trajectory \\
$[A, B]_{\re}$ & Experiment-consistent set \\
$[A, B]_\mathcal{E}$ & Design-consistent set \\
$V(\rx_0) = \mathcal{K}_n(A, [\rx_0 \, B])$ & Visible subspace \\
$k = \dim(V(\rx_0))$ & Visible dimension \\
$P_V \in \mathbb{R}^{n\times k}$ & Orthonormal basis for $V(\rx_0)$ \\
$\xi(t) = P^\top \rx(t)$ & Minimal (reduced) coordinates \\
$A_V, B_V$ & Visible subsystem matrices \\
$A_W, A_*$ & Hidden and cross-coupling blocks \\
$\mathcal{G}_{[0,T]}$ & Finite-time Gramian of the joint regressor \\
$\mu_{\min}$ & Left-eigenvector alignment margin \\
$d_\mathrm{PBH}(A, B \,|\, \rx_0)$ & Fixed-experiment PBH margin \\
$\mathrm{REE}_\mathrm{full}$ & Full-system relative estimation error \\
$\mathrm{REE}_\mathrm{vis}$ & Visible-subsystem relative estimation error \\
\bottomrule
\end{longtable}

\end{document}